%% file: paper.tex
\documentclass[]{bytedance_seed}

\usepackage[toc,page,header]{appendix}

\usepackage{algorithm}
\usepackage{comment}
\usepackage{algpseudocodex}
\usepackage{multirow}
\usepackage{booktabs}
\usepackage{makecell}
\usepackage{subcaption}
\usepackage{fancybox}
\usepackage{fancyvrb}
\usepackage{multirow}
\usepackage{tabularx}
\usepackage{wrapfig}
\usepackage{rotating}
\usepackage{colortbl}
\usepackage{framed}
\usepackage{tcolorbox}


\input{packages/header}

\input{packages/math_commands}

\usepackage{listings}
\newcommand{\lstbg}[3][0pt]{{\fboxsep#1\colorbox{#2}{\strut #3}}}

\lstdefinelanguage{diff}{
  basicstyle=\ttfamily\small,
  morecomment=[f][\lstbg{red!20}]-,
  morecomment=[f][\lstbg{green!20}]+,
}
\lstdefinelanguage{diffpython}{
  language=diff,
  morekeywords={def, if, else, for, while, return, import, from, as, class, with, try, except, finally, raise, lambda, and, or, not, in, is, None, True, False},
  morecomment=[l]{\#},
  morestring=[b]",
  morestring=[b]',
}

\hypersetup{
  colorlinks   = true, %
  urlcolor     = blue, %
  linkcolor    = blue, %
  citecolor   = blue %
}

\AtBeginDocument{\setlength\abovedisplayskip{4pt}}
\AtBeginDocument{\setlength\belowdisplayskip{4pt}}

\usepackage{minitoc}

\newcommand{\ttt}[1]{\texttt{#1}}

\title{Knapsack RL: Unlocking Exploration of LLMs via Optimizing Budget Allocation}

\author[1,2]{Ziniu Li}
\author[2]{Congliang Chen}
\author[3]{Tianyun Yang}
\author[3]{Ding Tian}
\author[2,3]{Ruoyu Sun}
\author[1]{Ge Zhang}
\author[1]{Wenhao Huang}
\author[2,3]{Zhi-Quan Luo}

\affiliation[1]{ByteDance Seed}
\affiliation[2]{The Chinese University of Hong Kong, Shenzhen}
\affiliation[3]{Shenzhen Research Institute of Big Data}

\abstract{
Large Language Models (LLMs) can self-improve through reinforcement learning, where they generate trajectories to explore and discover better solutions. However, this exploration process is computationally expensive, often forcing current methods to assign limited exploration budgets to each task. This uniform allocation creates problematic edge cases: easy tasks consistently succeed while difficult tasks consistently fail, both producing zero gradients during training updates for the widely used Group Relative Policy Optimization (GRPO). We address this problem from the lens of exploration budget allocation. Viewing each task's exploration as an "item" with a distinct "value" and "cost", we establish a connection to the classical knapsack problem. This formulation allows us to derive an optimal assignment rule that adaptively distributes resources based on the model's current learning status. When applied to GRPO, our method increases the effective ratio of non-zero policy gradients by 20-40\% during training. Acting as a computational "free lunch", our approach could reallocate exploration budgets from tasks where learning is saturated to those where it is most impactful. This enables significantly larger budgets (e.g., 93 rollouts) for especially challenging problems, which would be computationally prohibitive under a uniform allocation. These improvements translate to meaningful gains on mathematical reasoning benchmarks, with average improvements of 2-4 points and peak gains of 9 points on specific tasks. Notably, achieving comparable performance with traditional homogeneous allocation would require about 2x the computational resources.
}

\date{\today}
\correspondence{Ziniu Li at \email{liziniu@bytedance.com} and Ge Zhang at \email{zhangge.eli@bytedance.com}}

\begin{document}
\maketitle

\section{Introduction}

The remarkable capabilities of Large Language Models (LLMs) have led to their widespread application across various domains \citep{openai2025gpt5,comanici2025gemini,anthropic2025claude,meta2025llama,yang2025qwen3,seed2025seed1}. While pre-training on vast text corpora endows LLMs with general knowledge and linguistic fluency, fine-tuning them for specialized tasks often necessitates more targeted optimization beyond pre-training. Reinforcement Learning (RL) has emerged as a powerful paradigm for this purpose \citep{ouyang2022training, li2024remax,guo2025deepseek}, enabling LLMs to iteratively self-improve by interacting with environments. A popular instantiation is RL with verifiable rewards \citep{lambert2024tulu}, where LLMs generate responses and receive binary (true/false) feedback based on their outcomes, iteratively refining their internal policies to search for optimal solutions. Initially pioneered in mathematical reasoning \citep{jaech2024openai}, this framework has since been extended to domains like coding \citep{deepcoder2025} and agentic tasks \citep{team2025kimi}.

A core challenge in these applications is \emph{exploration}—sampling diverse trajectories to find better solutions. This process is computationally expensive in practice due to sequential nature of autoregressive generation. As such, most RL pipelines use a small number of rollouts per prompt (e.g., 8) for exploration. However, this uniform allocation strategy could lead to some problematic outcomes. For example, in the Group Relative Policy Optimization (GRPO) \citep{shao2024deepseekmath} algorithm, meaningful learning signals (gradients) only emerge when both successful and failed attempts are present in the same batch. With a uniform budget, easy tasks often result in all-success outcomes, and hard tasks in all-failure outcomes, leading to zero gradients and stalled learning. This issue has been well-documented in previous research \citep{yu2025dapo, chen2025spectral}, and we approach it from the broader perspective of strategic exploration budget allocation.

\begin{figure}[t]
    \centering
    \includegraphics[width=0.9\linewidth]{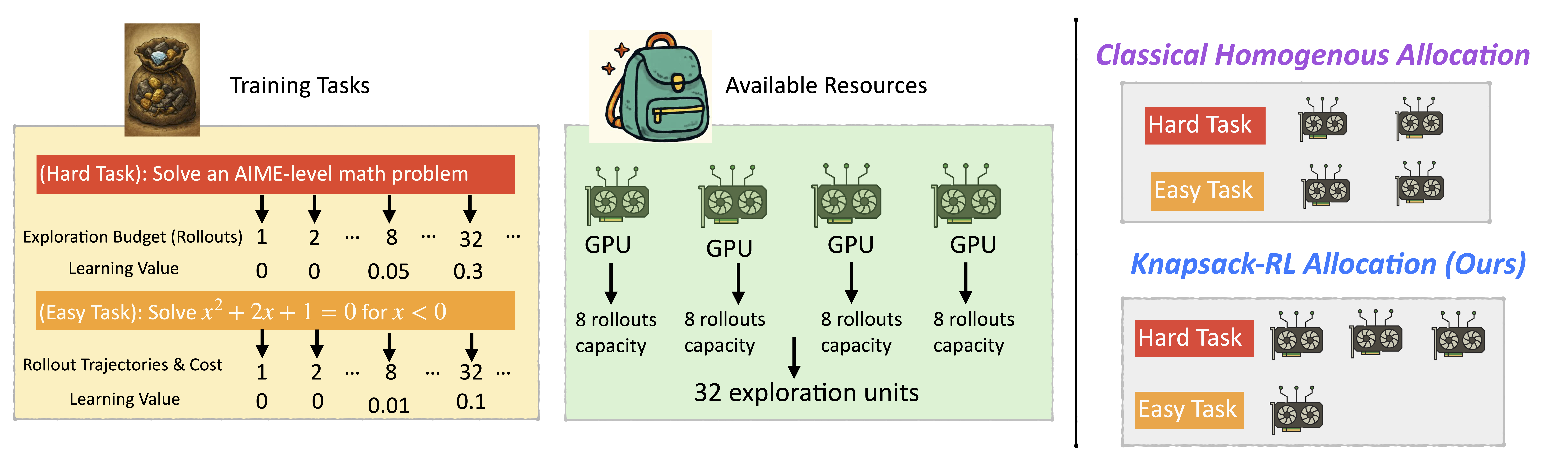}
    \caption{Illustration of our framework for allocating exploration budgets among tasks from computational resources. We model each task as an item with learning value and computational cost, then solve the allocation problem using Knapsack optimization.}
    \label{fig:main_framework}
\end{figure}

We argue the fundamental problem is the mismatch between a task's difficulty and its assigned exploration budget. Hard tasks, which require extensive  (could even require more than 100) exploration to find useful trajectories, receive too little effort under a uniform rule. Easy tasks, which require minimal exploration, waste compute by being over-sampled. Thus, a heterogeneous and customized exploration allocation strategy is preferred. 

To this end, we introduce a knapsack-based formulation: each task, when assigned a certain budget, can be conceptualized as an "item" with an associated value (learning potential) and cost (computational effort of exploration). The allocation problem is thus equivalent to the classical knapsack problem \citep{mathews1896partition,pisinger1998knapsack}, where the objective is to maximize total value under a fixed global budget. We refer to this approach as Knapsack RL; see \cref{fig:main_framework} for illustration. When applied to the popular GRPO framework, our method enables a dynamic, heterogeneous allocation of exploration budgets, which allows sufficient exploration on training tasks.

Empirically, across Qwen series models \citep{yang2024qwen2, yang2025qwen3} sized from 1B to 7B, we observe a 20-40\% improvement in effective gradient ratios, translating into more reliable policy improvements and average performance gains of about 2-4 points on several challenging benchmarks. To get a better sense of this improvement, we note that achieving comparable improvements with uniform allocation would require nearly 2x the computation. We present this as a proof-of-concept, demonstrating a promising direction to boost the effectiveness of RL.

\section{Preliminary}

Following \citep{ouyang2022training, shao2024deepseekmath}, we model language generation as autoregressive sampling from a conditional probability distribution $\pi_{\theta}(y|x)$, where $x$ represents the input prompt and $y$ represents the generated response. The parameter $\theta$ denotes the trainable parameters. Our goal is to improve the language model via RL by maximizing the expected performance of responses generated from the model distribution $\pi_{\theta}$:
\begin{align}\label{eq:reward_maximization}
\max_{\theta} \mathbb{E}_{y \sim \pi_{\theta}(\cdot|x)} [r(x, y)].
\end{align}
In this paper, we focus on RL with verifiable rewards \citep{lambert2024tulu}. Specifically, let $y = (\texttt{CoT}, \texttt{answer})$ denote the concatenation of Chain-of-Thought (CoT) \citep{wei2022chain} reasoning steps \texttt{CoT} and the final solution \texttt{answer}. The reward function $r(x, y)$ is defined as:
\begin{align} \label{eq:outcome_reward}
r(x, y) = \mathbb{I}(\texttt{answer} \text{ is correct with respect to } x),
\end{align}
where $\mathbb{I}(\cdot)$ is the indicator function and $r \in \{0, 1\}$ is binary (1 for correct, 0 for incorrect). This is equivalent to a $\{-1, 1\}$ scheme, where incorrect responses incur negative reward. This outcome-based reward formulation has been widely adopted (see e.g., \citep{guo2025deepseek} and references therein) and has been shown to effectively incentivize reasoning abilities \citep{wen2025reinforcement}.

\begin{algorithm}[htbp]
\caption{RL with Classical Homogeneous Budget Allocation} 
\label{algo:classical_rl} 
\begin{algorithmic}[1]
\For{iteration $t = 1, 2, \ldots$} 
\State Sample a mini-batch of prompts $(x_1, \ldots, x_M)$
\State Generate $N$ responses for each prompt $x_i$ \Comment{Budget Allocation}
\State Evaluate the rewards (e.g., \cref{eq:outcome_reward}) and compute the gradients (e.g., \cref{eq:grpo})
\State Update model parameters with estimated gradients
\EndFor 
\end{algorithmic} 
\end{algorithm}

To optimize \cref{eq:reward_maximization}, policy gradient methods \citep{sutton1999policy} are commonly employed. Among these, REINFORCE \citep{williams1992simple}-style stochastic policy gradient methods have become standard since the work of \citep{li2024remax}. These methods stochastically sample $N$ responses from $\pi_{\theta}$ and estimate gradients using direct reward feedback. Originally designed for single-task RL, this approach is typically extended to multi-task RL by employing homogeneous exploration budget allocation. \cref{algo:classical_rl} summarizes this classical framework.

In \cref{algo:classical_rl}, the sampling process in Line 3 corresponds to exploration in RL, where the model generates responses to search for optimal solutions. Line 5 corresponds to exploitation, updating the model to leverage feedback from data. We adopt the widely used gradient estimator from Group Relative Policy Optimization (GRPO) \citep{shao2024deepseekmath}:
\begin{align}
g(\theta) = \sum_{i=1}^{M} \sum_{j=1}^{N} \nabla_{\theta} \log \pi_{\theta}(y_{ij}|x_i) \cdot (r(x_i, y_{ij}) - b_{i}) \cdot c_{i}, \label{eq:grpo}
\end{align}
where $y_{ij}$ denotes the $j$-th sampled response for prompt $x_i$, and $\nabla_{\theta} \log \pi_{\theta}(y_{ij}|x_i)$ represents the gradient of the log-probability with respect to model parameters $\theta$. The baseline $b_i$ and normalization factor $c_i$ are defined as: $b_{i} = 1/N \cdot \sum_{j=1}^{N} r(x_i, y_{ij})$ and $c_i = 1/({\sigma_i + \epsilon})$ with $\sigma_i = \sqrt{1/N\cdot\sum_{j=1}^{N}(r(x_i, y_{ij}) - b_i)^2}$ is the standard deviation of rewards for prompt $x_i$, and $\epsilon$ is a small constant ($10^{-6}$) preventing division by zero when $\sigma_i = 0$. Technically, GRPO computes relative advantages within each response group (prompt), increasing likelihood of positive responses and decreasing likelihood of negative ones.

\section{Diagnosing Exploration in Homogeneous Budget Allocation}  \label{sec:diagnosing}

In this section, we discuss the limitations of homogeneous budget allocation for GRPO and present empirical observations that motivate our work.

\subsection{Motivation}

Exploration in RL is computationally expensive due to the sequential nature of autoregressive generation, often requiring substantial GPU memory and hours of computation, especially for reasoning tasks. From the sample efficiency perspective, it is critical to assess how much each collected sample actually contributes to gradient updates. For GRPO, we make the following observation.
\begin{obs}
Let $g_i = \sum_{j=1}^{N} \nabla_{\theta} \log \pi_{\theta}(y_{ij}|x_i) \cdot (r(x_i, y_{ij}) - b_{i}) \cdot c_{i}$ be the gradient for prompt $i$. If $\sigma_i = 0$, meaning that all $N$ sampled responses for $x_i$ yield identical rewards (all correct or all incorrect), then $(r(x_i, y_{ij}) - b_i) = 0$ for every sample in this group, leading to $g_i = 0$. In this case, the model receives no learning signal from that prompt.
\end{obs}
This phenomenon is widely recognized as a major bottleneck for GRPO in practice \citep{yu2025dapo,chen2025spectral}. To formally track it, we introduce the metric \texttt{effective-gradient-ratio}, which measures the proportion of individual samples that contribute non-zero gradients:
\begin{align}  \label{eq:effective_gradient_ratio}
\texttt{effective-gradient-ratio} = \frac{1}{M \cdot N} \sum_{i=1}^{M} \sum_{j=1}^{N} \mathbb{I}(g_{i, j} \neq 0),
\end{align}
where $g_{i,j} = \nabla_{\theta} \log \pi_{\theta}(y_{ij}|x_i) \cdot (r (x_i, y_{ij}) - b_{i}) \cdot c_{i}$ is the gradient contribution from the $j$-th sample of the $i$-th prompt. Higher values indicate that more samples provide useful gradient signals. We also define two complementary metrics: \ttt{zero-gradient-ratio (by all positive rewards)}: proportion of prompts yielding zero gradients due to uniformly positive rewards; and \ttt{zero-gradient-ratio (by all negative rewards)}: proportion of prompts yielding zero gradients due to uniformly negative (or zero) rewards.

We visualize these dynamics in Figure~\ref{fig:qwen2.5_7b_gradient_ratio} for the \texttt{Qwen2.5-Math-7B} model trained on the \texttt{DAPO-MATH-17K} dataset. Each mini-batch contains $M = 256$ prompts with $N=8$ rollouts per prompt. The results reveal several concerning patterns:

\textbf{Low Overall Effectiveness:} The effective gradient ratio consistently remains below 60\%, meaning that over 40\% of sampled data fails to contribute to model updates—a significant waste of computational resources.

\textbf{Dynamic Training Phases:} The gradient dynamics exhibit three distinct phases:
\begin{itemize}
    \item  \textbf{Early Training (0-70 iterations, approximately the first epoch):} The model struggles with most tasks, leading to predominantly all-negative rewards (green line peaks near 95\%). This results in minimal learning signals being generated.
    \item \textbf{Mid Training (70-600 iterations):} As the model improves, it begins solving some tasks while still failing others, creating the mixed outcomes necessary for effective gradients. The effective gradient ratio can maintain above 40\% during this phase.
    \item \textbf{Late Training (600+ iterations):} Tasks become increasingly easy, leading to a rise in all-positive rewards (orange line increases to 40\%). Simultaneously, challenging tasks still result in all-negative rewards (the green line fluctuates around 20\%). As a result, the effective-gradient-ratio steadily decreases to about 20\% by 1000 iterations.
\end{itemize}

\begin{wrapfigure}[12]{r}{0.45\linewidth}
\begin{center}
    \centering
    \includegraphics[width=\linewidth]{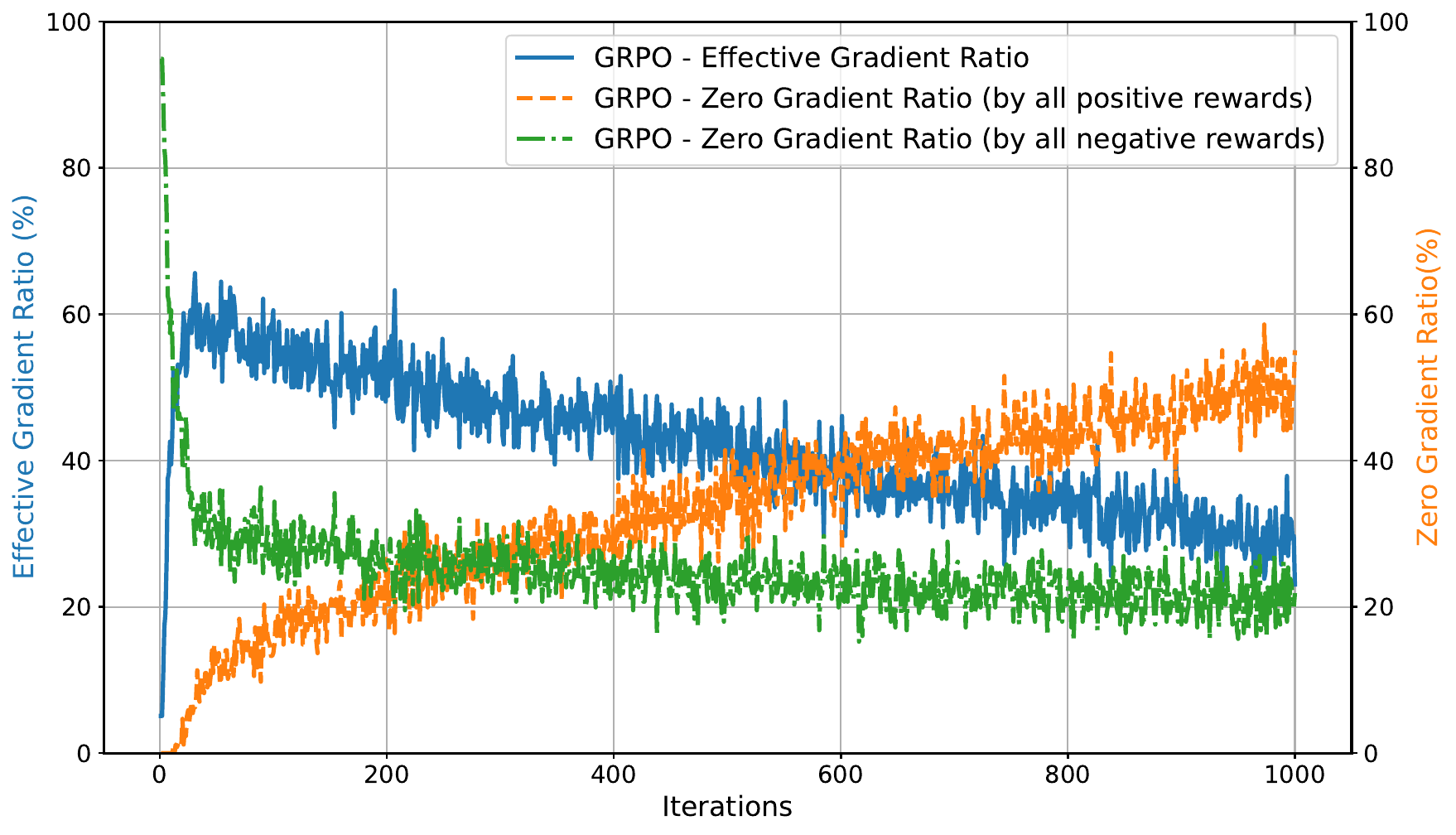}
    \caption{The ratio of effective gradients and zero gradients during training.}
    \label{fig:qwen2.5_7b_gradient_ratio}
\end{center}
\end{wrapfigure}

We provide theoretical analysis toward understanding the above empirical observations in the next section.

\subsection{Theoretical Analysis}
\label{sec:analysis}

We model reward outcomes as Bernoulli random variables to analyze the  exploration budget required.

\begin{defn}[Success Rate]
We define the success rate $p_i$ on a prompt $x_i$ as the probability that the model generates a correct response: $p_i \equiv  p(x_i) = \mathbb{E}_{y \sim \pi_{\theta}(\cdot|x_i)}[r (y|x_i)] = \Pr[r(y|x_i) = 1]$.
\end{defn}
This formulation allows statistical analysis of stochastic gradients. For $N$ sampled responses $y_{i1}, \ldots, y_{iN}$ on a prompt $x_i$, the probability that both correct and incorrect samples are observed is:
\begin{align*}
    \sP(g_{i} \neq 0) &= 1 - \sP[\text{all rewards are the same}] = 1 - \sP[\text{all rewards are 1's}] - \sP[\text{all rewards are 0's}] \\
    &= 1 - p_i^{N} - (1-p_i)^{N}.
\end{align*}
This raises the question: how large must the sampling budget $N$ be to obtain a non-zero gradient? We answer this from two perspectives: high-probability guarantees and expected sample complexity.

\newpage 
\begin{thm}[Exploration Budget]  \label{lem:budget_required}
Given a prompt $i$ with the success rate $p_i \in (0, 1)$, we have that 
\begin{itemize}
    \item \textbf{High probability bound:} For any $\alpha \in (0, 1)$, to ensure $\sP(g_i \ne 0) \geq \alpha$, it suffices to take $N \gtrsim \frac{\ln(1 - \alpha)}{\ln(\max\{p_i, 1-p_i\})}$.
    \item \textbf{Expected number of rollouts:} Let $N^{\operatorname{first}}_i$ denote the number of independent rollouts required until $g_i \neq 0$ is achieved for the first time. Its expectation is: $\expect[N^{\operatorname{first}}_i] = 1/p_i + 1 / (1-p_i) - 1$.
\end{itemize}
\end{thm}
Please refer to Appendix \ref{appendix:proof} for the proof. To illustrate, for example, if $p=0.5$, we need $3$ rollouts on average to obtain a non-zero gradient. For a hard task with $p=0.01$, we require $100$ rollouts on average, and to achieve a $90\%$ chance of non-zero gradient, we would need $229$ rollouts.

We show the theoretical predictions in \cref{fig:success_rate_vs_budget}. We employ the   \texttt{Qwen2.5-Math-7B-Instruct} model to generate $256$ responses for $1{,}000$ prompts from the \texttt{DAPO-Math-17K} dataset. Then we estimate $p$ and compute the minimal budget $N$ needed for $g_i \neq 0$ from the data. We exclude prompts that with empirical success rate of $0.0$ or $1.0$, because our exploration budget $256$ is not sufficient. The results show that a typical budget of $N=8$ only covers tasks with $p \in [0.1, 0.9]$. For tasks with $p \approx 0$ or $p \approx 1$, even increasing $N$ to $16$ or $32$ is insufficient. Overall, our analysis shows that the sampling budget required for meaningful gradients could be much larger than what is practically used. This also helps explain the low effective gradient ratio observed in \cref{fig:qwen2.5_7b_gradient_ratio}.

\begin{figure}[t]
    \centering
    \includegraphics[width=0.83\linewidth]{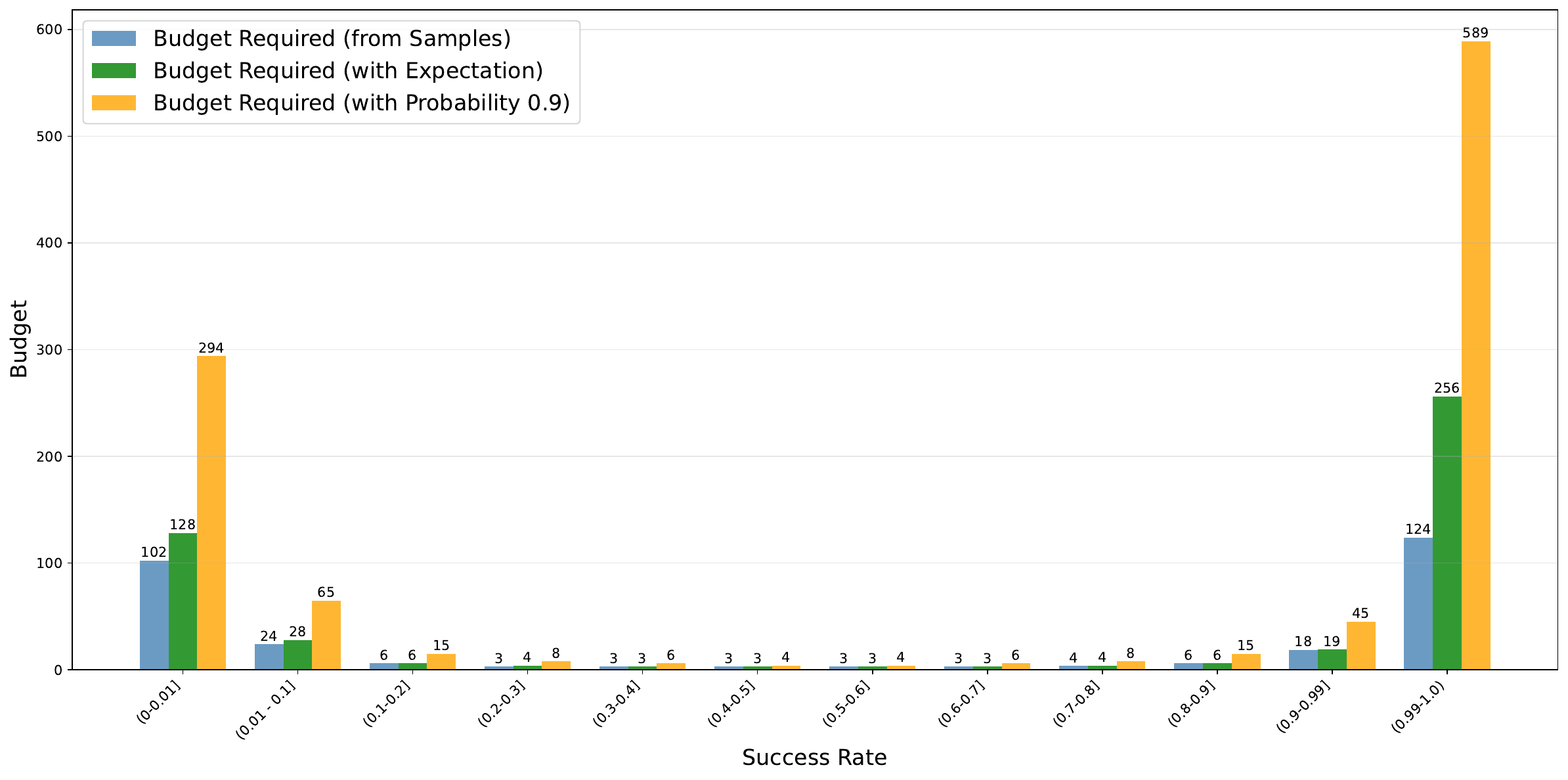}
    \caption{Exploration budget required to ensure non-zero gradients based on success rate. Note that success rates with in the same bins are grouped from real samples, which may not be symmetry, rendering the exploration budget may not be symmetry as the theory suggests.}
    \label{fig:success_rate_vs_budget}
\end{figure}

Existing practices typically address this kind of insufficient exploration challenge in two ways:
\begin{itemize}
    \item \textbf{Increasing the exploration budget uniformly.} This involves raising $N$—for example, from $8$ to $16$ or even $32$—which could help address exploration on extremely hard or easy tasks and improve the effective gradient ratio. However, setting a very large value for $N$, such as $N=100$, is often impractical due to prohibitive computational costs.
    \item \textbf{Filtering hard and easy prompts.} Tasks that are too easy or too hard are dropped. This kind of approach is leveraged in \citep{team2025kimi,yu2025dapo}. However, as we have seen, the proportion of prompts yielding zero gradients due to all-negative rewards (the green line in Figure \ref{fig:qwen2.5_7b_gradient_ratio}) is around 20\% in late training, indicating many tasks are not yet fully solved. If we simply filter these prompts, we may close off a crucial source for RL, where meaningful learning often comes from converting failures into successes. That is, removing hard prompts deprives the model of opportunities to practice on challenging examples, limiting the information available to LLMs.
\end{itemize}
In this work, we favor addressing this issue by scaling exploration budgets, but recognize that the first approach presents a fundamental \emph{computation-exploration dilemma}. This tension motivates our pursuit of a more principled solution for allocation of exploration budgets.

\section{Proposed Approach: Knapsack-based Budget Allocation}

In this section, we introduce our approach to addressing the {exploration–computation dilemma}. Our key principle is that computational resources are typically fixed by user constraints; thus, we do not assume access to additional resources. Instead, we treat the available compute as a fixed pool and design a centralized allocation strategy that distributes exploration budgets across tasks in a more informed way.

The central technical question is: \textbf{given a fixed total budget, what is the optimal allocation for RL exploration?} The homogeneous allocation strategy fails to consider the following aspects to be optimal:
\begin{itemize}
    \item \emph{Exploration cost}: some tasks require more rollouts to make progress.  
    \item \emph{Learning value}:  the potential benefit of improving performance on a given task.
\end{itemize}
We formalize the above idea as a constrained optimization problem:
\begin{align} 
\max_{N_1, \ldots, N_M} &\sum_{i=1}^{M} \operatorname{Value}(N_i, p_i)  \label{eq:knapsack_rl} \\
\text{subject to} & \sum_{i=1}^{M} N_i = N_{\operatorname{total}}, \quad  N_{\operatorname{low}} \leq N_i \leq N_{\operatorname{up}}, \quad N_i \in \mathbb{Z}^+,  \nonumber
\end{align}
where $N_i$ is the number of trajectories allocated to prompt $x_i$, and $p_i$ is the success rate. The bounds $N_{\operatorname{low}}$ (e.g., 2) and $N_{\operatorname{up}}$ (e.g., 128) allow to enforce coverage and prevent degenerate allocations. The total budget $N_{\operatorname{total}}$ is usually set to $N \times M$ to match the homogeneous allocation rule.

This formulation is structurally identical to the {classical knapsack problem} \citep{pisinger1998knapsack}. In the analogy, a task equipped with a specific exploration budget corresponds to an item. The number of allocated rollout trajectories $N_i$ is the item’s weight, and the resulting $\operatorname{Value}(N_i, p_i)$ is the item’s value. This analogy is natural: increasing $N_i$ requires more computation (longer generation time or more GPU memory), but it may also yield a higher learning benefit by producing more informative gradients. The knapsack’s capacity corresponds to the total exploration budget, which is physically determined by the available computational resources. The optimization objective is thus to maximize the total learning value subject to the fixed exploration budget. We summarize the correspondence in \cref{tab:connection_with_knapsack}.

\begin{table}[htbp]
    \centering
    \caption{Connection between RL exploration and the knapsack problem.}
    \label{tab:connection_with_knapsack}
    \begin{tabular}{c|c}  
    \toprule
     \textbf{RL Exploration Setting} & \textbf{Knapsack Analogy} \\  
     \midrule
     A task with an exploration budget & An item placed in the knapsack \\
     Number of trajectories assigned to the task & Weight of the item \\
     Potential learning benefit of allocating budget to the task & Value of the item \\
     Total exploration budget across all tasks & Knapsack capacity \\
     Available computational resources (e.g., GPUs) & The knapsack (the container itself) \\
     \bottomrule
    \end{tabular}
\end{table}

A subtle but important point is that task difficulty alone does not determine value. In other words, a task $x_i$ (an item) does not inherently possess value based solely on its difficulty $p_i$. Instead, value emerges only when difficulty is paired with a particular budget $N_i$, which governs its effective contribution. Consequently, the optimization problem is not defined over $M$ tasks directly but over all valid \emph{task–allocation pairs}. This means the effective item set is larger, with size $M \times (N_{\operatorname{up}} - N_{\operatorname{low}})$, since each task can correspond to multiple potential items depending on the allocated budget.

\subsection{Formulation of Task Value} \label{sec:task_value}

In this section, we substantiate the above framework with the proposed idea. We recognize that homogeneous budget allocation fails to take the task value into consideration. For GRPO, we address this issue by defining the value of assigning $N_i$ exploration budget units to prompt $x_i$ as
\begin{align*}
\operatorname{Value}(N_i, p_i) = \operatorname{ProbNonZeroGradient}(N_i, p_i) \times \operatorname{InfoGain}(p_i),
\end{align*}
where $\operatorname{ProbNonZeroGradient}(N_i, p_i) = 1 - p_i^{N_i} - (1-p_i)^{N_i}$ is the probability of obtaining a non-zero gradient (see \cref{sec:analysis}) for GRPO, and $\operatorname{InfoGain}(p_i)$ quantifies the informativeness of a gradient if one occurs. It can also be extended to other algorithms; see Appendix \ref{appendix:extension}. Our design emphasizes coverage of effective gradients across prompts: it accounts for whether a non-zero gradient is likely to appear, but not for the exact balance of positive versus negative samples.

In this work, we define $\operatorname{InfoGain}$ as a measure of the expected increase in success probability after a gradient update, while noting that alternative formulations could be explored in future work. Formally, let $p^t_i$ denote the success rate before the update and $p^{t+1}_i$ the rate after the update. We define
\begin{align*}
\operatorname{InfoGain} &= \Delta p_i = p^{{t+1}}_i - p^t_i.
\end{align*}
Directly computing this requires access to the post-update success probability, which is intractable. We address this issue by some approximation bounds.

\begin{prop}  \label{prop:info_value}
With the Taylor expansion, the $\operatorname{InfoGain}$ can be approximated by $\boxed{p_i(1-p_i)^2}$.
\end{prop}

Please refer to Appendix \ref{appendix:proof} for detailed derivation.  This approximation relies only on the current success rate $p_i$. Although idealized and not exact, it captures essential intuitions while remaining simple to compute. Its key properties are:
\begin{itemize}
\item $\operatorname{InfoGain}(p_i)$ is maximized at $p_i = 1/3$. This aligns with the intuition that uncertain-but-promising samples are most valuable.
\item $\operatorname{InfoGain}(p_i)$ is asymmetric: for equally distant values of $p_i$ from $1/3$, harder tasks yield larger information gain than easier tasks. Furthermore, $\operatorname{InfoGain}(p_i) \to 0$ as $p_i \to 0$ or $p_i \to 1$, meaning extremely hard or extremely easy prompts provide diminishing value.
\end{itemize}

\begin{wrapfigure}[16]{r}{0.42\linewidth}
\begin{center}
    \centering
    \includegraphics[width=\linewidth]{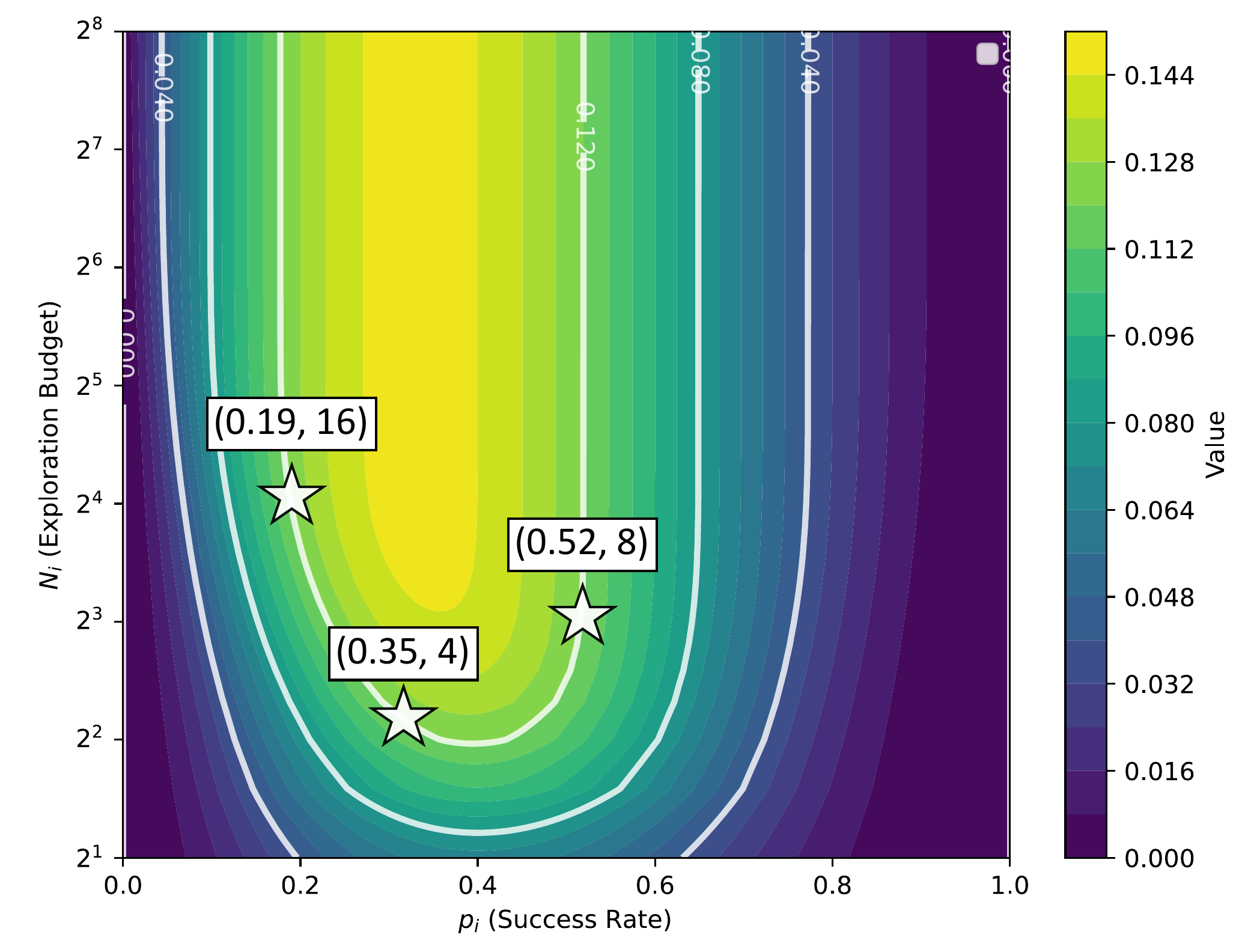}
    \caption{The interplay between success rate, exploration budget and the value.}
    \label{fig:value_vs_budget}
\end{center}
\end{wrapfigure}

We visualize our defined $\operatorname{Value}(N_i, p_i)$ in \cref{fig:value_vs_budget}. This contour plot shows lines of equal value, highlighting the interplay between the success rate $p_i$ and the exploration budget $N_i$. For example, the three highlighted points demonstrate that different combinations of $p_i$ and $N_i$ can yield comparable high values. A task with the success rate $p_i=0.35$ (which is close to $1/3$), requires a relatively small exploration budget of $N_i=4$ to achieve a high value. However, for tasks with success rates further from this optimum, such as a harder task with $p_i=0.19$ or an easier one with $p_i=0.52$, the required exploration budget is now specified as $N_i=16$ or $N_i=8$ respectively to reach the same value level.

\subsection{Algorithm Implementation}
\label{subsec:practical_considerations}

\textbf{Estimating success rates.} In practice, the success rate $p_i$ is not directly available as a prior and must be estimated from collected samples. In this work, we employ a simple heuristic: using the success rates observed in the previous epoch as estimates for the current one. Specifically, the first epoch may follow a homogeneous budget allocation rule, after which the proposed knapsack-based approach leverages the estimated success rates $\widehat{p}_i$ to guide allocation. Although this strategy introduces some delay and noise, it has proven empirically effective. More sophisticated estimation techniques (e.g., online logistic regression) that account for task correlations present promising directions for future improvement.

These estimated $\widehat{p}_i$ values are directly used to formulate the discrete constrained optimization problem (Equation \ref{eq:knapsack_rl}), which can be solved in polynomial time using standard dynamic programming techniques. With Numba \citep{lam2015numba} acceleration, it typically runs within 1–2 seconds. 

\textbf{Handling Extreme Cases.} Our value function defined in \cref{sec:task_value} assigns a zero value to prompts with empirical success rates of $0$ or $1$, which would otherwise lead to zero budget allocation for these prompts. To prevent their complete exclusion and maintain coverage:
\begin{itemize} 
    \item For $\widehat{p}_i = 1.0$ (prompts always solved correctly), the estimate may be not accurate from history samples, so we allocate a small minimum budget (e.g., $2$) to ensure they are still considered. This can be achieved by set $N_{\operatorname{low}}$ in \cref{eq:knapsack_rl}.
    \item For $\widehat{p}_i = 0.0$ (prompts never solved correctly), we employ a fallback allocation strategy. We first estimate the total budget required for prompts with $p_i \in (0,1]$ according to \cref{lem:budget_required} and the above rule. Any remaining budget is subsequently distributed among extremely hard tasks. This strategy is particularly beneficial in later training stages where many prompts become easy, thus freeing up capacity to focus on hard tasks.
\end{itemize}

\textbf{Rollout Balancing.} In practice, the total number of trajectories ($M \times N$) is typically generated by $W$ parallel workers (where $W < M$), often leveraging efficient inference engines like vLLMs \citep{kwon2023efficient}. While a homogeneous allocation rule allows for a simple division of $M$ prompts among $W$ workers (each performing $N$ rollouts per prompt), our knapsack-based approach can lead to significant imbalance in allocated rollouts per prompt. This occurs because certain prompts may be allocated disproportionately large exploration budgets, creating an uneven workload and potentially leading to GPU idles and inefficient resource utilization.

To address this issue, we employ a simple rollout balancing strategy: we treat each allocated rollout for a prompt as an individual execution job. These execution jobs are then randomly dispatched to the available workers, with the inference engine generating one response per prompt. This approach is suitable for settings where prompts are not excessively long, thus not strictly requiring advanced techniques like prefix caching. For scenarios involving longer prompts, we would consider using the Karmarkar–Karp bin-packing algorithm \citep{karmarkar1982efficient} to group prompts into approximately balanced batches based on their allocated budgets. Workers would then process these balanced groups of prompts, potentially utilizing prefix caching.

Overall, our knapsack-based exploration method integrates seamlessly into large-scale RL training pipelines with minimal modifications (see Listing \ref{lst:knapsack_rl} in the Appendix). Computationally, it adds negligible overhead. Algorithmically, it introduces no additional hyperparameters to tune and does not bias policy gradients. From a systems perspective, core components of inference (e.g., vLLM-based accelerated generation \citep{kwon2023efficient}) and training (e.g., FSDP \citep{zhao2023pytorch} and Megatron \citep{megatron-lm}) remain unchanged, ensuring full compatibility with existing infrastructure.

\section{Experiments}

\subsection{Main Results}

\textbf{Experiment Setting.} We implement Knapsack-RL and baseline methods using the large-scale RL training framework \ttt{Verl} \citep{sheng2025hybridflow}. Our primary focus is GRPO \citep{shao2024deepseekmath}, a widely examined method, and we refer to our specific implementation as \emph{Knapsack-GRPO}. Training utilizes the {DAPO-Math-17K} dataset \citep{yu2025dapo}, which comprises 17,917 prompts, each with a ground truth answer for verification.

We conduct experiments with both pre-trained and instruction-tuned models. The pre-trained models include Qwen3-4B-Base \citep{yang2025qwen3} and Qwen2.5-Math-7B \citep{yang2024qwen2}. For instruction-tuned models, we utilize DeepSeek-R1-Distill-Qwen-1.5B \citep{guo2025deepseek} (abbreviated as DPSK-R1-Distill-1.5B) and Qwen3-4B-Instruct-2507 \citep{yang2025qwen3} (abbreviated as Qwen3-4B-Instruct). In each iteration, we employ a mini-batch size of $M = 256$ prompts and generate $N = 8$ rollouts. Our models are trained for 1,000 iterations.  The extensive training duration of 1,000 iterations for Qwen2.5-Math-7B, for example, requires about 1,400 GPU hours with A100 GPUs.

For evaluation, we follow \citep{deepscaler2025} and assess our method on several mathematical reasoning benchmarks: AIME, AMC, MATH, MINERVA, and OLYMPIAD Bench (OLYMPIAD for short). Given AIME's small sample size, we combine its 2024 and 2025 editions into a single dataset, hereafter referred to as AIME. Additionally, we include GPQA \citep{rein2023gpqa} as an out-of-domain evaluation, which tests scientific reasoning across physics, chemistry, and biology. All reported performance metrics are averaged over 16 generated responses.

\begin{table}[htbp]
    \centering
    \caption{Evaluation performance (\ttt{avg@16}) comparison across different models and benchmarks.}
    \label{tab:main_results}
    \begin{tabular}{l|ccccccc}
    \toprule 
                & {\footnotesize \rotatebox{45}{AIME}} & { \footnotesize \rotatebox{45}{AMC}} & { \footnotesize \rotatebox{45}{MATH}} & {\footnotesize \rotatebox{45}{MINERVA} } & {\footnotesize \rotatebox{45}{OLYMPIAD} } & { \footnotesize  \rotatebox{45}{GPQA} } &  {\footnotesize \rotatebox{45}{Avg} } \\
    \midrule
    DPSK-R1-Distill-1.5B     &  25.3 & 62.1 & 81.4 & 25.8 &  41.7 & 39.1 & 42.9 \\
    + GRPO  & 27.6 & 71.1 & 84.0 & 27.6 & 46.4 & 36.7 &  45.9 \\ 
    \rowcolor{blue!10} + Knapsack-GRPO & \textbf{34.0}  & 	\textbf{75.1} & \textbf{86.7} &	\textbf{28.5}	& \textbf{49.7} & 	\textbf{40.3}	& \textbf{49.7} \\  \hline 
        Qwen3-4B-Base & 6.6 & 29.9 & 48.0 & 19.4 & 23.1 & 26.4 & 22.9 \\
    + GRPO & 20.7 & 56.9 & 80.6 & 31.9 &  44.9 & \textbf{46.6} & 43.2 \\
     \rowcolor{blue!10} + Knapsack-GRPO & \textbf{20.8} & \textbf{66.0} & \textbf{81.0} & \textbf{35.7} & \textbf{46.2} & 45.5 & \textbf{45.1} \\
    \hline 
    Qwen3-4B-Instruct  & 47.7 & 82.5 & 92.4 & 35.4 & 61.6 & 43.0 & 58.6 \\
    + GRPO & 47.0 & \textbf{84.9} & \textbf{92.5} & \textbf{41.8} & 61.8 & 54.4 & 59.2 \\
    \rowcolor{blue!10} + Knapsack-GRPO& \textbf{48.2} & 83.1 & \textbf{92.5} & 38.2 & \textbf{63.5} & \textbf{59.9} & \textbf{61.9} \\
    \hline
    Qwen2.5-Math-7B  & 12.3 & 41.0 & 61.2 & 11.8 & 26.1 & 22.0 & 26.7 \\
    + GRPO & 23.9  & 70.6 & 81.7 & 33.6 & 41.9 & 40.8 & 45.2 \\
    \rowcolor{blue!10} + Knapsack-GRPO & \textbf{24.3} & \textbf{77.4} & \textbf{83.9} & \textbf{34.5} & \textbf{44.1} & \textbf{43.8} & \textbf{47.5}  \\  \bottomrule 
    \end{tabular}%
\end{table}

We report the evaluation performance in \cref{tab:main_results}, observing consistent improvements across all tested models after applying our RL training. Specifically, Knapsack-GRPO consistently outperforms GRPO. For instance, in terms of average performance, it improves by 3.8 points for DPSK-R1-Distill-1.5B compared to GRPO. On specific benchmarks, the improvements are even more significant: for example, 6.4 points on AIME for DPSK-R1-Distill-1.5B, 9.1 points on AMC for Qwen3-4B-Base, 5.5 points on GPQA for Qwen3-4B-Instruct, and 6.8 points on AMC for Qwen2.5-Math-7B.

\subsection{Understanding Knapsack-based Exploration}

This section delves into understanding the superiority of knapsack-based exploration. We visualize its exploration budget distribution and analyze its efficacy through gradient effectiveness and task status dynamics during training, primarily focusing on the Qwen2.5-Math-7B model. We provide additional results in Appendix \ref{sec:additional_results}.

\begin{wrapfigure}[17]{r}{0.5\linewidth}
\begin{center}
    \centering
    \includegraphics[width=\linewidth]{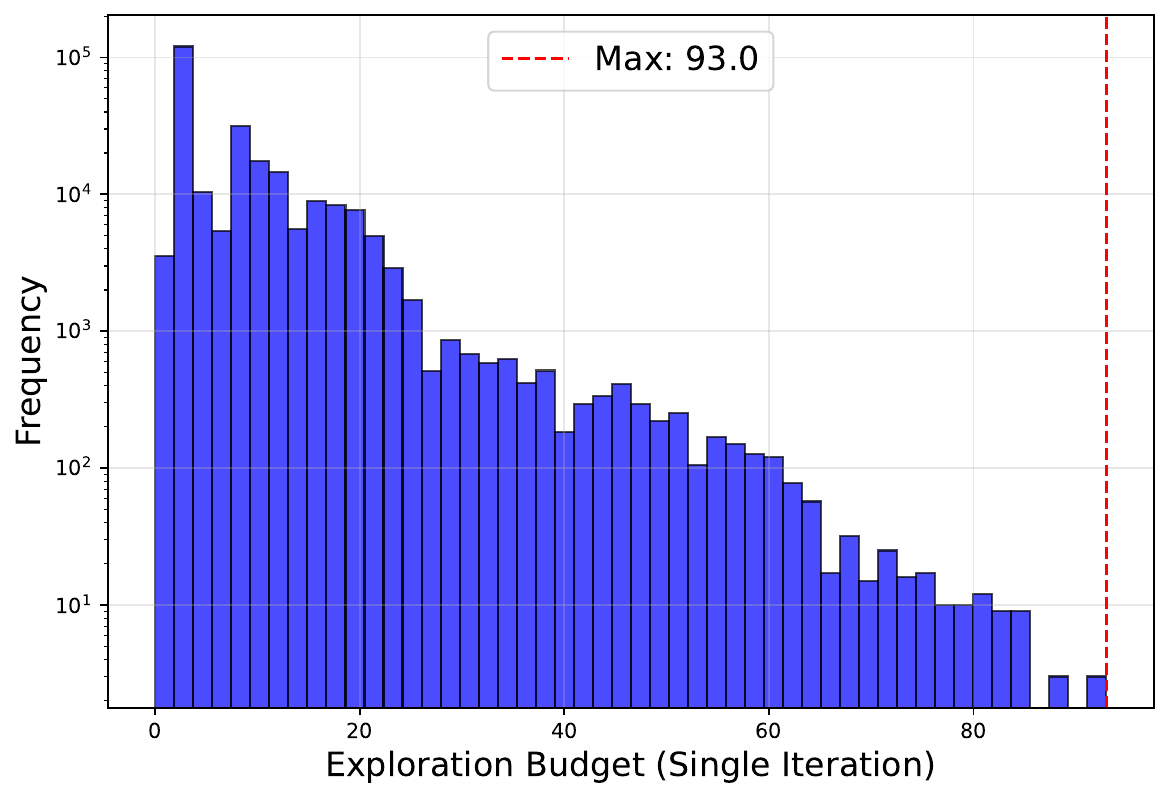}
    \caption{Distribution of exploration budgets allocated by knapsack-GRPO for Qwen2.5-Math-7B during training.}
    \label{fig:exploration_budget}
\end{center}
\end{wrapfigure}

\textbf{Exploration Budgets.} To illustrate the impact of knapsack-based exploration, we visualize the assigned exploration budgets. Specifically, we quantify the frequency with which different exploration budgets are allocated to individual prompts during the training. These results are presented in \cref{fig:exploration_budget}. We observe that, even without introducing additional computational resources, our approach can dynamically assign up to 93 exploration budgets to certain tasks. This level of dynamic, high-budget allocation is impractical to achieve under a conventional homogeneous budget allocation framework.

\textbf{Effective Gradient Ratio.} Figure \ref{fig:effective_gradient_ratio} shows the effective gradient ratio during training, as defined in \cref{eq:effective_gradient_ratio}. Knapsack-based budget allocation improves this ratio by approximately 20-40\% across models. Unlike uniform allocation, the knapsack method avoids a clear decreasing trend. This stems from dynamically distributing exploration budgets, targeting tasks with mixed successful and failed trajectories. These observations partially explain Knapsack-GRPO's policy improvements.

\begin{figure}[htbp]
    \centering
    \begin{subfigure}{0.32\linewidth}
      \centering
      \includegraphics[width=\linewidth]{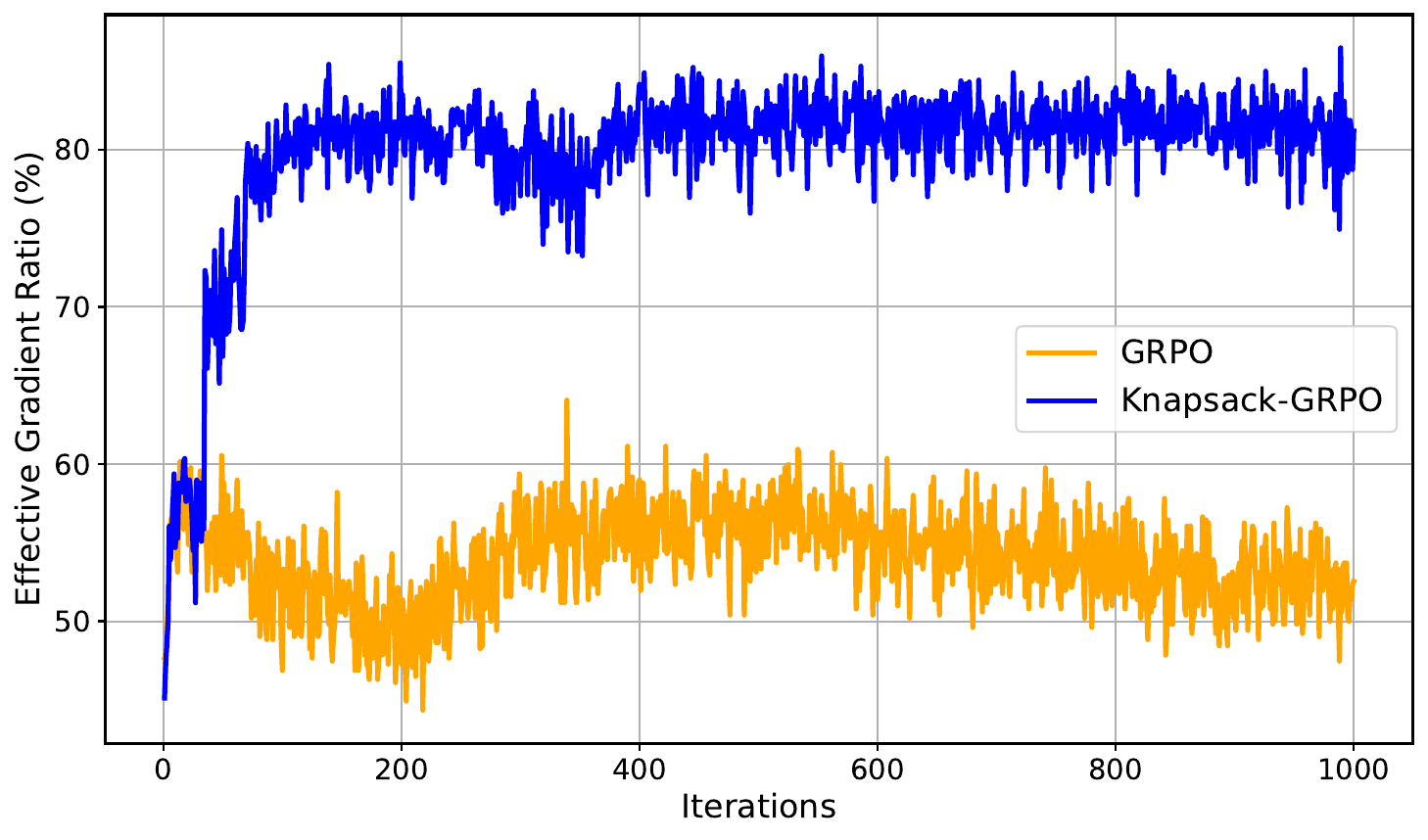} 
      \caption{DPSK-R1-Distill-1.5B}
    \end{subfigure}
    \hfill
    \begin{subfigure}{0.32\linewidth}
      \centering
      \includegraphics[width=\linewidth]{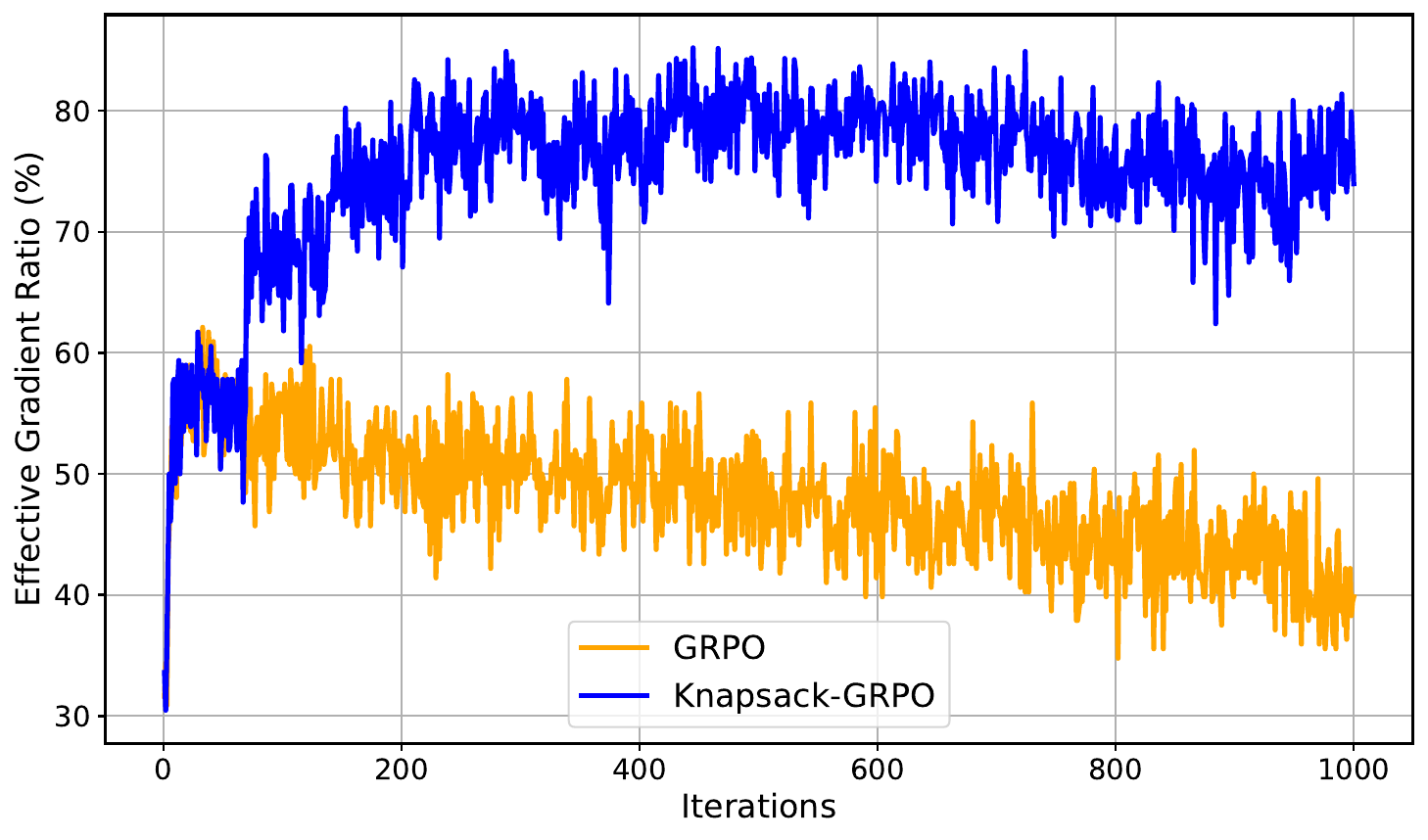}  
      \caption{Qwen3-4B}
    \end{subfigure}
    \begin{subfigure}{0.32\linewidth}
      \centering
      \includegraphics[width=\linewidth]{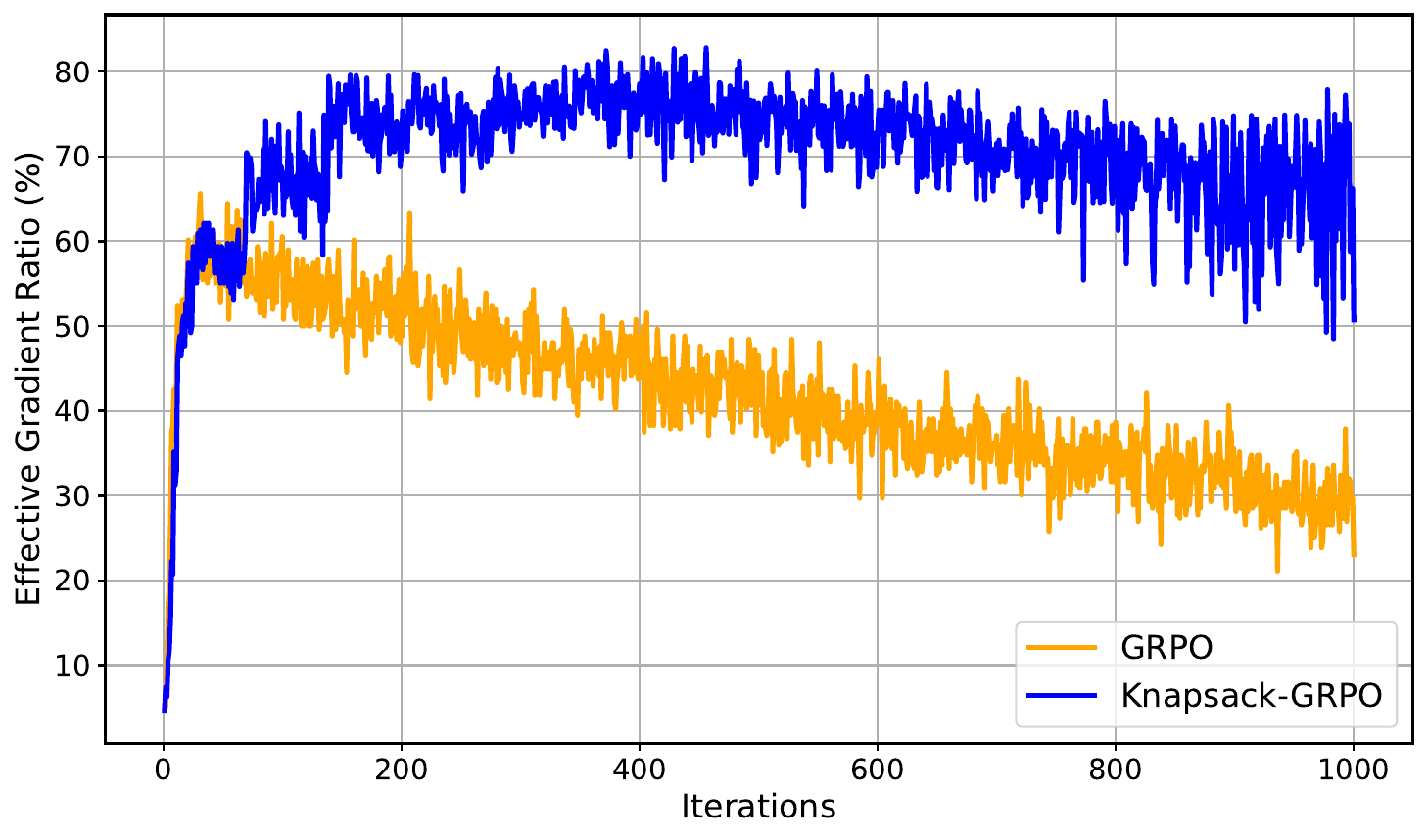}  
      \caption{Qwen2.5-Math-7B}
    \end{subfigure}
    \caption{Effective gradient ratio during training.}
    \label{fig:effective_gradient_ratio}
\end{figure}

\textbf{Task Transition Dynamics.}  To understand our method's influence on learning, we analyze prompt evolution during training. We categorize training prompts into five performance statuses based on success rate ($p_i$): \ttt{extremely-hard} ($p_i = 0$, all failures), \ttt{hard} ($0 < p_i \le 0.2$), \ttt{medium} ($0.2 < p_i < 0.8$), \ttt{easy} ($0.8 \le p_i < 1.0$), and \ttt{extremely-easy} ($p_i = 1.0$, all successes). Our analysis covers two aspects: 1) prompt status transitions after training, and 2) final prompt status distribution.

\begin{figure}[htbp]
    \centering
    \begin{subfigure}{0.45\linewidth}
      \centering
      \includegraphics[width=\linewidth]{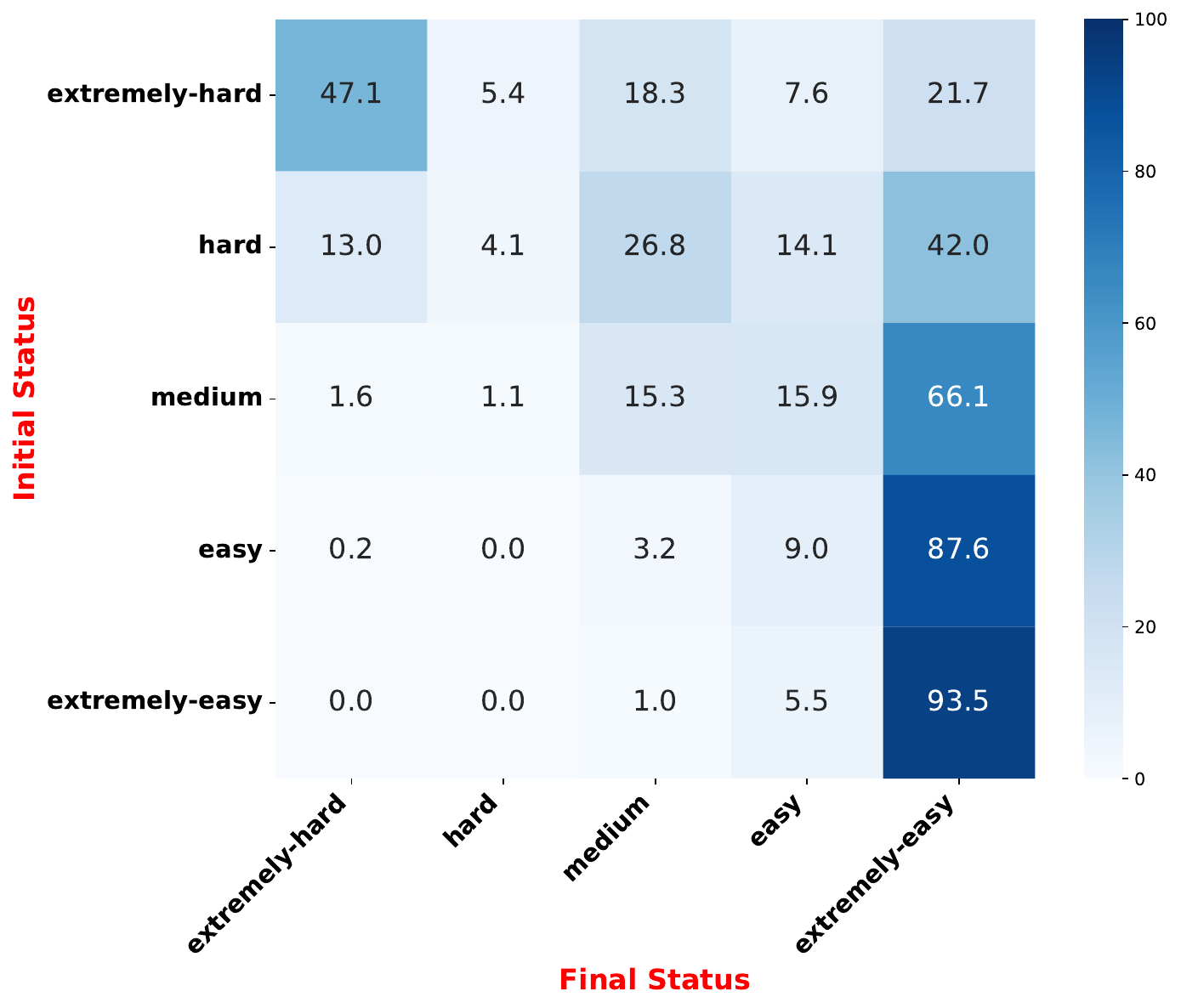} 
      \caption{GRPO}
    \end{subfigure}
    \hfill
    \begin{subfigure}{0.45\linewidth}
      \centering
      \includegraphics[width=\linewidth]{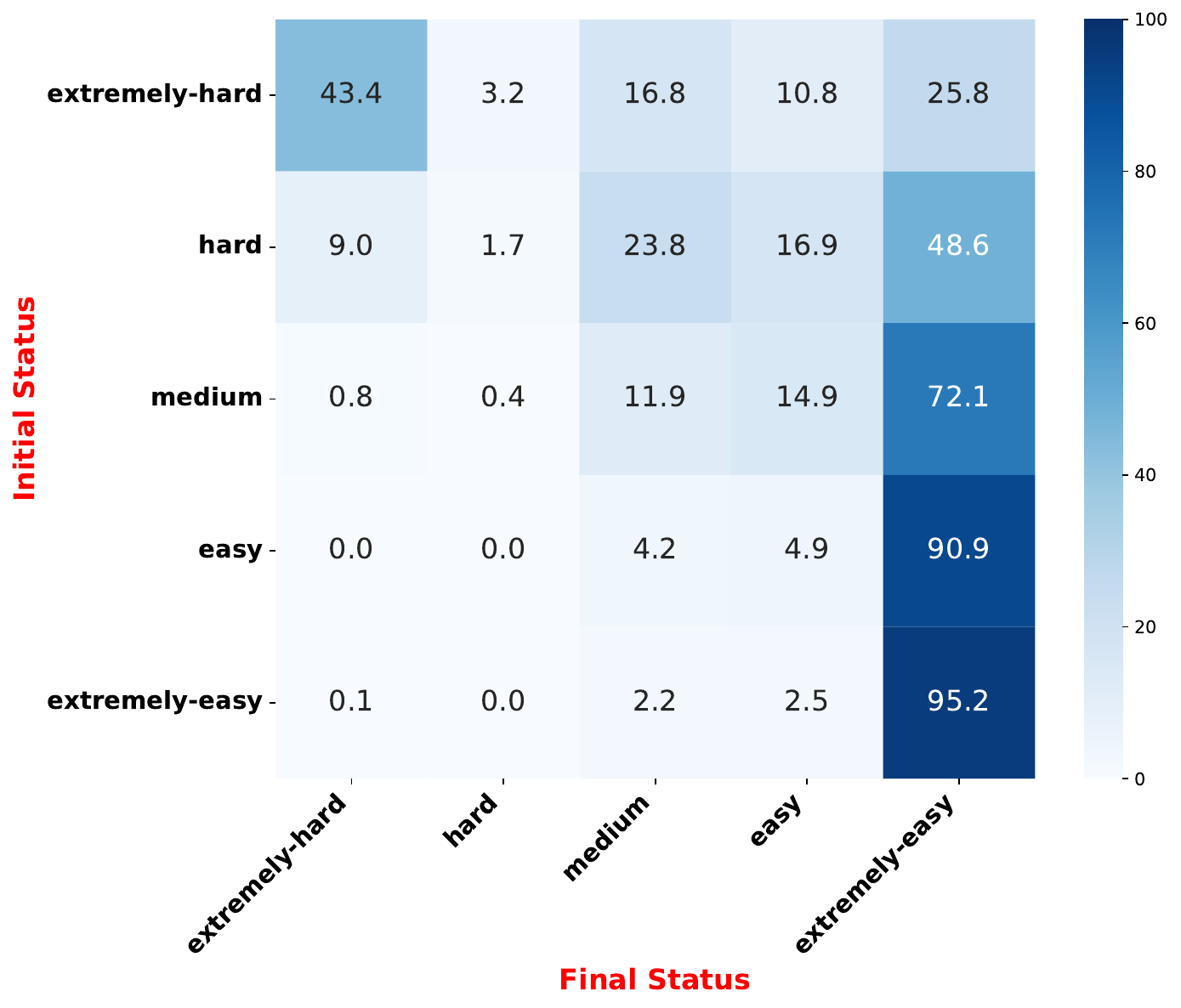}  
      \caption{Knapsack-GRPO}
    \end{subfigure}
    \caption{Prompt transition matrices for Qwen2.5-Math-7B during training. The cell $(i, j)$ indicates the percentage of samples transitioning from status $i$ to status $j$.}
    \label{fig:data_transition}
\end{figure}

Figure \ref{fig:data_transition} visualizes the $5 \times 5$ transition matrix for Qwen2.5-Math-7B training, illustrating prompt category transitions. Knapsack-GRPO demonstrates superior efficiency in learning challenging tasks. Specifically, the self-absorption frequency for \ttt{extremely-hard} samples (prompts remaining in that status) is 43.4\% for Knapsack-GRPO, notably lower than GRPO's 47.1\%. Furthermore, Knapsack-GRPO shows a higher transition rate to \ttt{extremely-easy} tasks (last column in heatmap) than GRPO, indicating more effectively mastered samples.

We also examine the final distribution of prompt statuses after training, specifically by counting the training samples in each status, as depicted in Figure \ref{fig:final_status}. Knapsack-GRPO has 3,596 \ttt{extremely-hard} tasks, less than GRPO's 3,793. This 197-task reduction suggests Knapsack-GRPO's dynamic budget allocation makes them more tractable. Consistent with observed transitions, Knapsack-GRPO yields 9,274 \ttt{extremely-easy} tasks, surpassing GRPO's 8,676.

\begin{wrapfigure}[15]{r}{0.4\linewidth}
\begin{center}
    \centering
    \includegraphics[width=\linewidth]{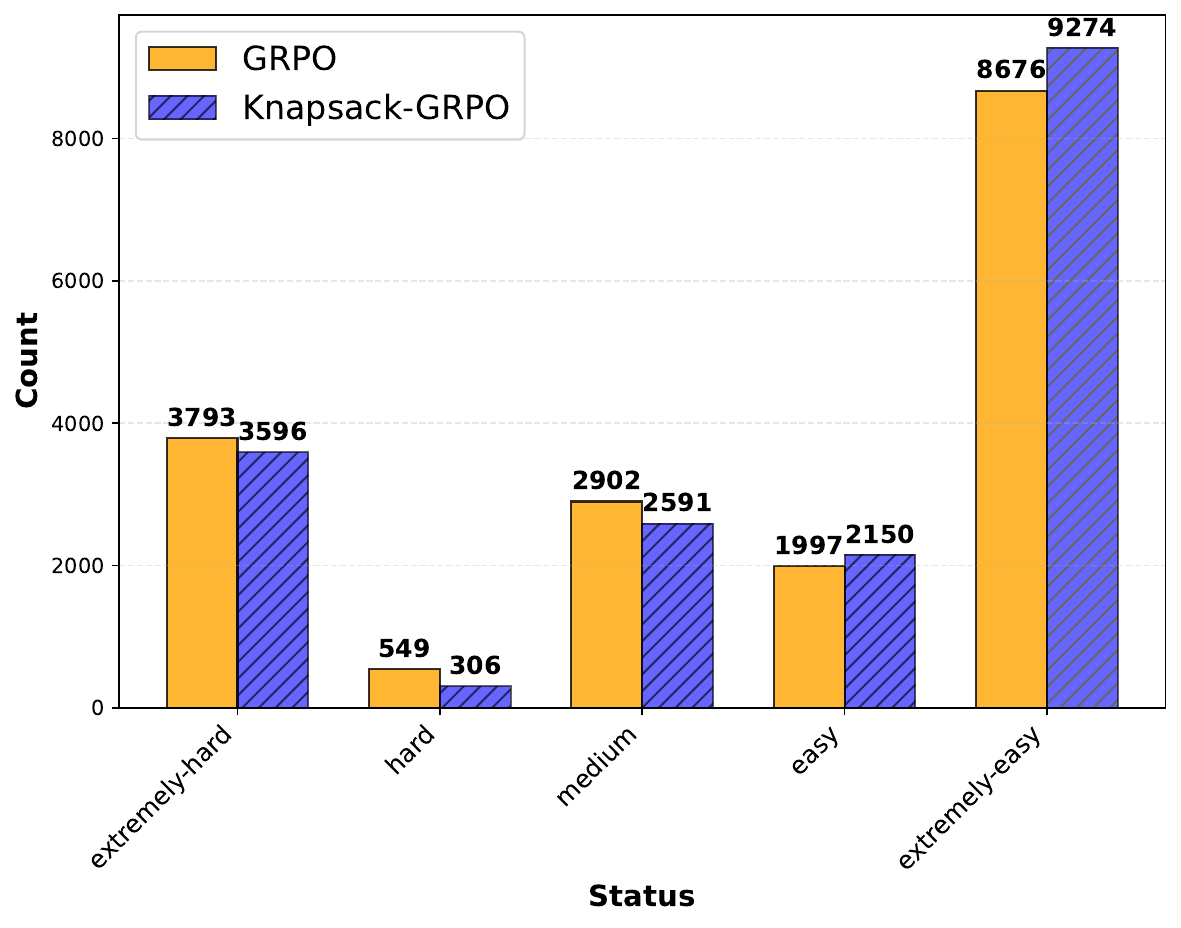}
    \caption{Distribution of sample statuses after training.}
    \label{fig:final_status}
\end{center}
\end{wrapfigure}

Despite these promising results, approximately 20\% of prompts remain in the \ttt{extremely-hard} category even after 1,000 training iterations. We investigate if these are truly unsolvable: for Knapsack-GRPO, 577 of these challenging prompts recorded at least one positive trajectory during optimization, implying they are not inherently unsolvable. Future research could explore experience replay techniques to address these samples more effectively.


\subsection{Experiments with Different Exploration Budgets}

Finally, we conduct experiments with varying total exploration budgets to assess performance under different computational resource constraints. In contrast to previous experiments, which used a total budget of $N_{\text{total}} = 256 \times 8 = 2048$, here we explore scenarios with $N_{\operatorname{total}} = 1024$ and $N_{\operatorname{total}} = 4096$. For the vanilla GRPO, this corresponds to using $N = 4$ and $N = 16$, respectively. Note that the total budget parameter $N$ does not impact Knapsack-GRPO in the same way, given its distinct allocation strategy.

\begin{wrapfigure}[13]{r}{0.4\linewidth}
\begin{center}
    \centering
    \includegraphics[width=\linewidth]{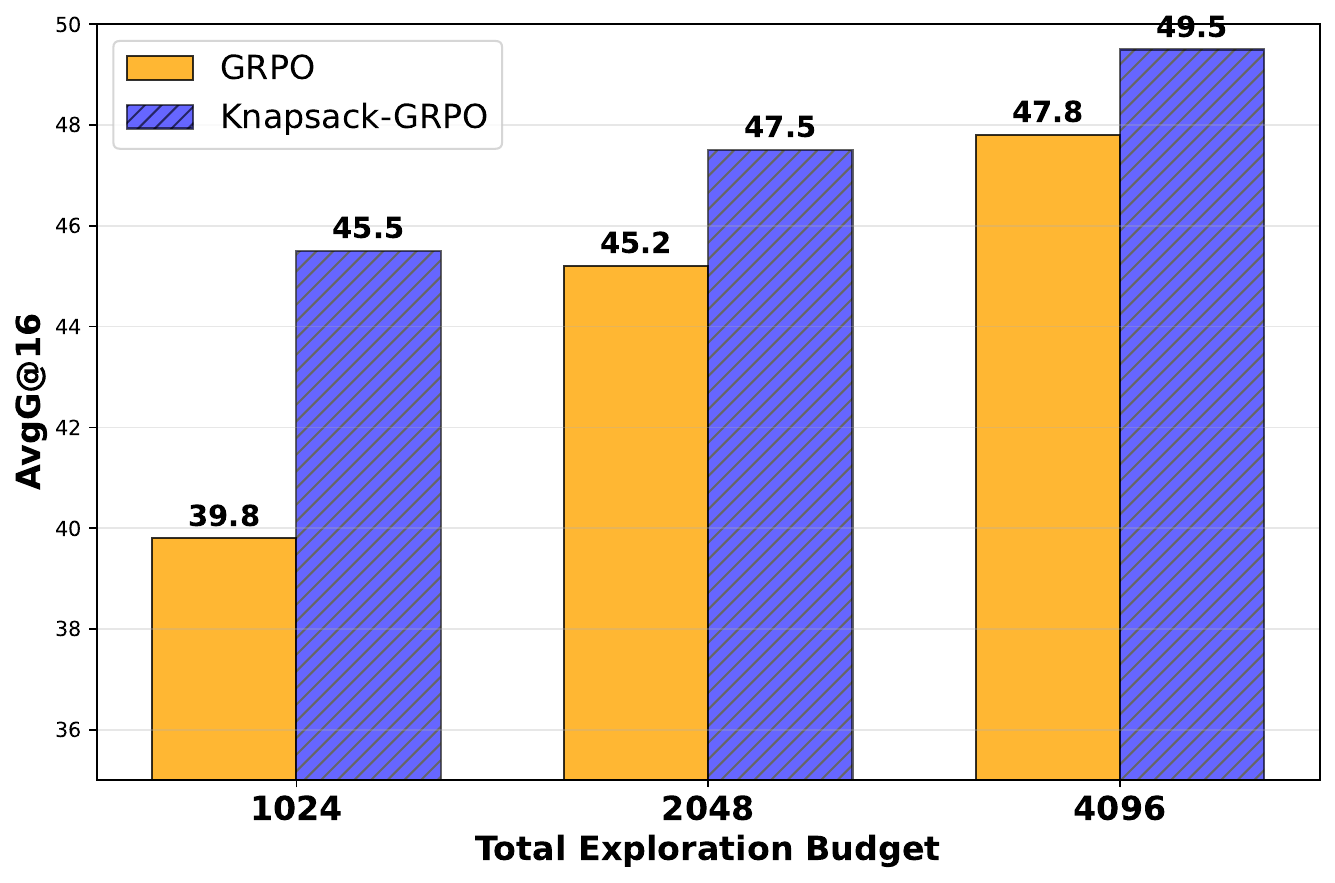}
    \caption{Performance comparison under different exploration budgets.}
    \label{fig:scaling_n}
\end{center}
\end{wrapfigure}

The results for the Qwen2.5-Math-7B model are shown in \cref{fig:scaling_n}. Knapsack-GRPO clearly outperforms GRPO, particularly when computational resources are limited. In the low-budget scenario, Knapsack-GRPO improves performance from 39.8 to 45.5, while continuing to maintain its advantage even with higher exploration budgets. These findings demonstrate that Knapsack-GRPO achieves the same performance as standard GRPO with roughly 2x the computational resources, highlighting the efficiency of its exploration budget allocation.

\section{Related Work}

\textbf{Data heterogeneity.} A central challenge in RL for LLMs arises from the heterogeneity of training data. Prompts vary substantially in difficulty, leading to diverse reward distributions, and these distributions evolve throughout training, further complicating learning. Prior work has recognized this issue: for example, \citet{li2024remax} observed substantial variations in reward distributions across prompts, which complicated stable gradient estimation. Their solution introduced refined baselines to reduce variance, thereby improving the \emph{exploitation} stage of RL. Following this, many advanced policy optimization methods have been proposed (e.g., \citep{shao2024deepseekmath, ahmadian2024back, yu2025dapo}). We refer readers to surveys \citep{zhang2025survey, wang2025survey} for a broader overview. By contrast, our work directly tackles the \emph{exploration} challenge posed by heterogeneous data, focusing on how to allocate exploration resources more effectively to capture informative trajectories in the first place.

\textbf{Prompt selection and curriculum learning.} Another line of research seeks to improve data efficiency through prompt selection and curriculum design \citep{lin2024rho, zhang2024policy, chen2025self, sun2025prepo}. For example, \citet{chen2025self} used advantage estimates as a proxy for difficulty to construct curricula, while \citet{sun2025prepo} employed perplexity as a difficulty metric. Additionally, \citet{yu2025dapo} presented the concept of "dynamic sampling" to address the issue of sparse gradients; however, it is crucial to clarify that their "sampling" refers to selecting prompts that yield effective gradients, rather than dynamically allocating exploration budgets.

These works primarily operate along the axis of \emph{which prompts to train on}, while maintaining homogeneous exploration per prompt. In contrast, our approach focuses on \emph{how much exploration to allocate} to each prompt. Our work aims to addresses the need for more extensive exploration on challenging tasks directly. To underscore this fundamental difference, consider that prompt selection methods might prioritize a \emph{subset} of simple prompts to achieve a high effective gradient ratio. Our work, however, aims to dynamically design the exploration budget for \emph{all} prompts to ensure that each receives sufficient exploration to generate effective gradients, especially those that are inherently harder.

\textbf{Resource allocation.} Optimizing resource allocation has long been studied in operations research and systems \citep{katoh2013resource, hussain2013survey}. However, connections to RL have been more limited, partly because traditional RL often addresses single-task settings where computational budgets do not require explicit distribution. Theoretical works have considered online exploration under knapsack constraints \citep{chen2020efficient, brantley2020constrained}, but typically in the single-task formulation. By contrast, our formulation of {Knapsack RL} explicitly allocates exploration budgets across multiple tasks from a centralized computational pool.

The most closely related work is \citep{yao2025optimizing}, which also investigates dynamic resource allocation but within rejection sampling and RAFT \citep{dong2023raft} frameworks, focusing on variance reduction. By contrast, our method directly targets online RL, formulating exploration allocation as a knapsack optimization problem that explicitly balances computational cost with expected learning value. In addition, studies such as \citep{zhang2025scaling, wang2025every} consider computation allocation during the \emph{inference} stage, whereas our work addresses it in the \emph{training} stage where learning signals must be actively generated.

\textbf{Scaling RL.} Our method resonates with the principle of test-time scaling \citep{snell2024scaling, brown2024large}, which allocates additional computational resources (e.g., best-of-$N$ sampling, majority voting) to improve response quality. Similarly, our approach leverages extra compute to amplify exploration, thereby enhancing the quality of collected training signals. More broadly, our work aligns with recent efforts that scale compute in post-training to unlock stronger downstream performance \citep{jaech2024openai, liu2025prorl}.

\section{Conclusion}

Motivated by the observation that RL agents require extensive exploration on challenging tasks to gather informative feedback and drive self-improvement, we investigate the problem of optimally allocating computational resources for exploration. We formulate this problem as a knapsack optimization, where each task-budget pair is treated as an item with an associated cost and value. This framework enables us to prioritize harder tasks, thereby yielding more effective gradients and leading to superior policy improvements. This comes at no additional computational cost, effectively offering a "free lunch". 

We view this work as an initial step toward scaling exploration as a means to unlock RL’s potential in LLM post-training. Several directions offer promising avenues for future research:
\begin{itemize}
\item \textbf{Extending beyond rollout counts.} In this work, we focus on exploration budgets measured by the number of rollouts. Future extensions could incorporate other computational factors, such as the token length of responses or the number of interaction turns required in agentic tasks.
\item \textbf{Designing richer value functions.} We model the task value using a first-order Taylor expansion of one-step policy improvement. Exploring alternative formulations of value functions could yield more accurate assessments of learning potential and further enhance allocation strategies.
\item \textbf{Incorporating advanced exploration strategies.} Our study employs simple on-policy rollouts for tractability. While effective, some tasks remain unsolved under this paradigm. More sophisticated strategies—such as tree-based exploration inspired by Monte Carlo Tree Search used in AlphaGo \citep{silver2016mastering}—offer a compelling path forward. Recent work in Tree-RL \citep{hou2025treerl} and TreePO \citep{li2025treepo} demonstrates the promise of techniques like state rollbacks to advantageous intermediates. Integrating such methods with our knapsack-based allocation framework is promising.
\end{itemize}

\bibliographystyle{plainnat}
\bibliography{main}

\clearpage

\beginappendix

\input{appendix/appendix}

\end{document}

%% file: packages/header.tex
\usepackage{url}            %
\usepackage{booktabs}       %
\usepackage{multirow}    
\usepackage{amsfonts}       %
\usepackage{nicefrac}       %
\usepackage{microtype}      %
\usepackage{natbib}
\usepackage{enumerate}
\usepackage{hhline}
\usepackage{makecell}
\usepackage{pifont}

\usepackage{graphicx} %
\usepackage{amsmath}
\usepackage{amsthm}
\usepackage{amssymb}
\usepackage{tikz}
\usepackage{xcolor}
\usetikzlibrary{arrows}

\allowdisplaybreaks

\usepackage{mathrsfs}

\usepackage{hyperref}
\usepackage{bm}

\allowdisplaybreaks

\newcommand{\expect}{\mathbb{E}}

\newtheorem{thm}{Theorem}

\newtheorem{obs}{Observation}
\newtheorem{prop}{Proposition}

\newtheorem{defn}{Definition}

\crefname{thm}{Theorem}{Theorems}
\crefname{lem}{Lemma}{Lemmas}
\crefname{obs}{Observation}{Observations}
\crefname{cor}{Corollary}{Corollaries}
\crefname{prop}{Proposition}{Propositions}
\crefname{asmp}{Assumption}{Assumptions}
\crefname{defn}{Definition}{Definitions}
\crefname{oracle}{Oracle}{Oracles}
\crefname{fact}{Fact}{Facts}
\crefname{conj}{Conjecture}{Conjectures}
\crefname{rem}{Remark}{Remarks}
\crefname{example}{Example}{Examples}
\crefname{condition}{Condition}{Conditions}
\crefname{exercise}{Exercise}{Exercises}
\crefname{algorithm}{Algorithm}{Algorithms}
\crefname{table}{Table}{Tables}
\crefname{figure}{Figure}{Figures}
\crefname{section}{Section}{Sections}
\crefname{subsection}{Section}{Sections}
\crefname{appendix}{Appendix}{Appendices}
\crefname{message}{Message}{Messages}

\definecolor{red}{rgb}{1, 0, 0}

\definecolor{green}{rgb}{0, 1, 0}

\definecolor{blue}{rgb}{0, 0, 1}

\definecolor{orange}{rgb}{1, 0.4, 0.0}

%% file: appendix/appendix.tex
\appendix

\input{appendix/implementation}
\input{appendix/proofs}

\input{appendix/extension}
\input{appendix/experiment_details}

\input{appendix/additional_results}

%% file: appendix/proofs.tex
\section{Proof}
\label{appendix:proof}

\begin{proof}[Proof of \cref{lem:budget_required}]  

We prove both parts of the lemma.

\paragraph{Part 1: High probability bound.} We want to find the minimum $N$ such that $\sP(g_i \neq 0) \geq \alpha$ for a given $\alpha \in (0,1)$. From the problem setup, we have:
\begin{align*}
\sP(g_i \neq 0) &= 1 - p_i^N - (1-p_i)^N,
\end{align*}
For the condition $\sP(g_i \neq 0) \geq \alpha$ to hold, we require:
\begin{align*}
1 - p_i^N - (1-p_i)^N &\geq \alpha \\
p_i^N + (1-p_i)^N &\leq 1 - \alpha. \label{eq:constraint}
\end{align*}
Let $q = \max\{p_i, 1-p_i\}$. Since $p_i \in (0,1)$, we have $q \geq \frac{1}{2}$.  Without loss of generality, assume $p_i \geq \frac{1}{2}$, so $q = p_i$ and $1-p_i \leq p_i$. The case $p_i < \frac{1}{2}$ follows by symmetry.

Since $(1-p_i) \leq p_i$, we have $(1-p_i)^N \leq p_i^N$ for $N \geq 1$. Therefore:
\begin{align*}
p_i^N + (1-p_i)^N \leq 2p_i^N = 2q^N.
\end{align*}
For large $N$, the term $q^N$ dominates $(1-q)^N$ since $q > \frac{1}{2}$. More precisely, we have:
\begin{align}
\lim_{N \to \infty} \frac{(1-q)^N}{q^N} = \lim_{N \to \infty} \left(\frac{1-q}{q}\right)^N = 0,
\end{align}
since $({1-q})/{q} < 1$.

Therefore, for sufficiently large $N$, the constraint \eqref{eq:constraint} is dominated by the term $q^N$:
\begin{align}
q^N &\lesssim 1 - \alpha \quad  \Longleftrightarrow \quad N \ln q \lesssim \ln(1 - \alpha).
\end{align}
Since $q < 1$, we have $\ln q < 0$, which gives:
\begin{align*}
\boxed{N \gtrsim \frac{\ln(1 - \alpha)}{\ln q} = \frac{\ln(1 - \alpha)}{\ln(\max\{p_i, 1-p_i\})}}.
\end{align*}

\paragraph{Part 2: Expected number of rollouts.}  
Let $X_1,X_2,\dots$ be i.i.d.\ Bernoulli random variables with $\Pr(X_i=1)=p\in(0,1)$, where $1$ denotes ``success'' and $0$ denotes ``failure''. Define
\[
N^{\mathrm{first}}\equiv N = \min\{n\ge 1 : \text{both $0$ and $1$ have appeared among }X_1,\dots,X_n\}.
\]

We compute $\mathbb{E}[N]$ by conditioning on the first trial $X_1$.

\medskip
\noindent\textbf{Case 1: $X_1=1$ (probability $p$).}  
After the first success, we still need to wait until the first failure occurs.  
The waiting time for the first failure follows a geometric distribution with success probability $1-p$, whose expectation is $1/(1-p)$.  
Thus
\[
\mathbb{E}[N \mid X_1=1] = 1 + \frac{1}{1-p}.
\]
\noindent\textbf{Case 2: $X_1=0$ (probability $1-p$).}  
By symmetry, we wait for the first success; its waiting time has expectation $1/p$, so
\[
\mathbb{E}[N \mid X_1=0] = 1 + \frac{1}{p}.
\]
\medskip
Applying the law of total expectation:
\begin{align*}
\mathbb{E}[N]
&= p\Bigl(1+\frac{1}{1-p}\Bigr) + (1-p)\Bigl(1+\frac{1}{p}\Bigr) \\[4pt]
&= 1 + \frac{p}{1-p} + \frac{1-p}{p} \\[4pt]
&= \frac{1}{p} + \frac{1}{1-p} - 1.
\end{align*}
Hence, the expected number of rollouts until we first observe both a success and a failure is
\[
\boxed{\displaystyle \mathbb{E}[N^{\mathrm{first}}] = \frac{1}{p} + \frac{1}{1-p} - 1}.
\]
This completes the proof of the second part of Lemma~\ref{lem:budget_required}.
\end{proof}

\begin{proof}[Proof of Proposition \ref{prop:info_value}]
We provide a rigorous derivation under the following assumptions:
\begin{itemize}
    \item The policy follows a softmax distribution: $p_k = \frac{\exp(z_k)}{\sum_{j=1}^K \exp(z_j)}$ for action $k$.
    \item The gradient update follows the policy gradient rule with advantage $A$:
    \begin{align}
        z_k \leftarrow z_k + \eta \cdot A \cdot \sI [k=y] \cdot \nabla_{z_k} \log p_y
    \end{align}
    where $\eta$ is the learning rate and $y$ is the chosen action.
    \item We assume unit learning rate ($\eta = 1$) and unit advantage ($A = 1$) for simplicity.
\end{itemize}

\textbf{Step 1: Taylor expansion.} 
For small parameter changes, the change in success probability can be approximated by:
\begin{align*}
    \Delta p_y \approx \sum_{k=1}^{K} \frac{\partial p_y}{\partial z_k}. \Delta z_k
\end{align*}
\textbf{Step 2: Computing partial derivatives.}
For the softmax probability $p_y = \frac{\exp(z_y)}{\sum_{j=1}^K \exp(z_j)}$, we have:
\begin{align*}
    \frac{\partial p_y}{\partial z_y} &= p_y(1-p_y), \quad \text{ and } \quad \frac{\partial p_y}{\partial z_k} = -p_y p_k, \quad \text{for } k \neq y.
\end{align*}
\textbf{Step 3: Determining parameter updates.}
Under the policy gradient update rule, we have:
\begin{align*}
    \nabla_{z_k} \log p_y = \sI[k=y] - p_k.
\end{align*}
Therefore, the parameter updates are:
\begin{align*}
    \Delta z_y &= \sI[y=y] - p_y = 1 - p_y, \\
    \Delta z_k &= \sI[k=y] - p_k = 0 - p_k = -p_k, \quad \text{for } k \neq y
\end{align*}
\textbf{Step 4: Computing InfoGain.}
Substituting the partial derivatives and parameter updates:
\begin{align*}
    \Delta p_y &= \frac{\partial p_y}{\partial z_y} \Delta z_y + \sum_{k \neq y} \frac{\partial p_y}{\partial z_k} \Delta z_k \\
    &= p_y(1-p_y) \cdot (1-p_y) + \sum_{k \neq y} (-p_y p_k) \cdot (-p_k) \\
    &= p_y(1-p_y)^2 + p_y \sum_{k \neq y} p_k^2
\end{align*}
\textbf{Step 5: Simplification under first-order approximation.}
For the first-order Taylor approximation to be accurate, we require small parameter updates. Under this condition, the cross-terms $\sum_{k \neq y} p_k^2$ are second-order in the update magnitude and can be neglected compared to the main term $p_y(1-p_y)^2$.

Therefore, we obtain:
\begin{align*}
    \operatorname{InfoGain} \approx \boxed{p_y(1-p_y)^2}.
\end{align*}
This completes the proof.

\end{proof}

\begin{wrapfigure}[11]{r}{0.33\linewidth}
\begin{center}
    \centering
    \includegraphics[width=\linewidth]{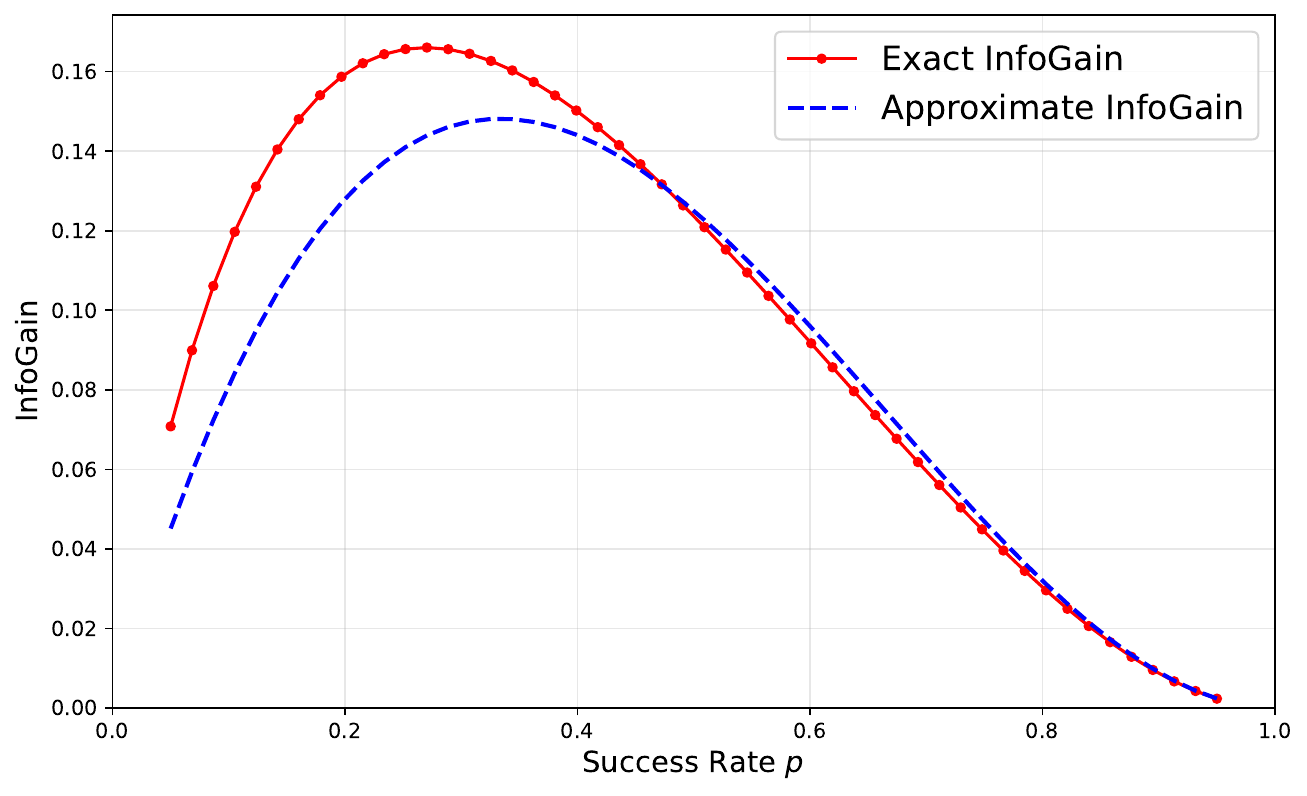}
   \caption{Comparison of exact \texttt{InfoGain} and approximate formula.}
    \label{fig:info_gain_approx}
\end{center}
\end{wrapfigure}

To validate this approximation, we conduct an empirical study with 100 actions, comparing the \texttt{InfoGain} computed through exact gradient updates against our theoretical approximation from \cref{prop:info_value}. As shown in \cref{fig:info_gain_approx}, the two curves align closely across different success rates, demonstrating that our formula $p(1-p)^2$ provides a reliable approximation for practical use.

%% file: appendix/extension.tex
\section{Extensions}
\label{appendix:extension}

In this work, we mainly focus on the widely used GRPO \citep{shao2024deepseekmath} algorithm to design the optimal allocation strategy. Here we discuss possible extensions for other RL algorithms by adapting the core framework while maintaining the same task value function structure:
\begin{align*}
\operatorname{Value}(N_i, p_i) = \operatorname{ProbNonZeroGradient}(N_i, p_i) \times \operatorname{InfoGain}(p_i).
\end{align*}
The key difference lies in how we compute $\operatorname{ProbNonZeroGradient}(N_i, p_i)$ for different algorithms:
\begin{itemize}
    \item \textbf{RLOO} \citep{ahmadian2024back}. RLOO's policy gradient estimator is equivalent to GRPO up to constants, thus we may not need fundamental changes. The probability of obtaining a non-zero gradient remains:
    \begin{align*}
        \operatorname{ProbNonZeroGradient}(N_i, p_i) = 1 - p_i^{N_i} - (1-p_i)^{N_i}.
    \end{align*}
    \item \textbf{ReMax} \citep{li2024remax}. ReMax leverages the reward of greedy response as baseline, rather than the averaged reward used in GRPO. In this setting, a gradient update occurs only when the sampled trajectory differs from the greedy response. If we denote the probability of the greedy response as $\alpha$, then the probability of sampling a trajectory different from the greedy response is $1-\alpha$. The probability of obtaining a non-zero gradient with $N_i$ samples becomes:
    \begin{align*}
        \operatorname{ProbNonZeroGradient}(N_i, \alpha) = 1 - \alpha^{N_i}.
    \end{align*}
    This represents the probability that at least one of the $N_i$ sampled trajectories differs from the greedy response, thereby producing a gradient signal.
    \item \textbf{REINFORCE} \citep{williams1992simple}. There is no baseline design in vanilla REINFORCE. We can directly calculate the $\operatorname{ProbNonZeroGradient}$ to account for the case where at least one trajectory receives a positive reward:
    \begin{align*}
        \operatorname{ProbNonZeroGradient}(N_i, p_i) = 1 - (1 - p_i)^{N_i}.
    \end{align*}
    This formulation is simpler than GRPO since we only need to ensure at least one successful trajectory occurs, rather than balancing positive and negative samples.
\end{itemize}

The proposed framework's modularity allows for straightforward adaptation to other RL algorithms by: (1) identifying the algorithm's gradient computation mechanism, (2) determining conditions for non-zero gradients, (3) calculating the corresponding $\operatorname{ProbNonZeroGradient}$ function, and (4) maintaining the same $\operatorname{InfoGain}(p_i) = p_i(1-p_i)^2$ formulation across algorithms. This demonstrates the general applicability of our value-based budget allocation approach beyond the specific GRPO implementation.

%% file: appendix/experiment_details.tex
\section{Experiment Details}
\label{appendix:experiment_details}

Our experiments utilized the large-scale RL training framework \ttt{Verl}, specifically version 0.5.0. No modifications were made to the core training and inference code, with the exception of the advantage calculation, where values were clipped between -5 and 5. This was implemented because, as rollout responses were scaled, we observed their values could become significantly large in extreme cases, thus requiring this additional clipping for numerical stability.

Following recommendations from \citep{yu2025dapo}, the learning rate was set to $10^{-6}$, with importance sampling clipping ratios (high/low) of $0.28$ and $0.2$, respectively. Neither KL nor entropy regularization was employed. Models were trained with a maximum sequence length of 4K tokens, with the exception of DPSK-R1-Distill-1.5B, which utilized 8K tokens to accommodate its typically longer Chain-of-Thought (CoT) behaviors requiring more context.

For evaluation results reported during training, models were assessed every 10 training iterations using 16 generated responses. To manage evaluation time, 100 evaluation samples were randomly selected from benchmarks when the total number of samples exceeded this number.

For the final evaluation performance presented in \cref{tab:main_results}, different maximum sequence lengths were used to prevent response truncation: 4K tokens for Qwen2.5-Math-7B, 8K tokens for Qwen3-4B and Qwen3-4B-Instruct, and 16K tokens for DPSK-R1-Distill-1.5B. Consequently, these results may not perfectly align with those reported in the training curves.

\begin{lstlisting}[
    caption={Python pseudo code implementation of knapsack RL. Two components are modified: (1) budget allocation is replaced with knapsack optimization for better resource distribution, and (2) task status is updated based on external feedback.},
    label={lst:knapsack_rl},
    abovecaptionskip=2pt,
    belowcaptionskip=7pt,
    language=diffpython,
]
def budget_allocation(batch, total_budget, **kwargs):
-    budget = np.full(len(batch), total_budget // len(batch['prompt']))
+    budget = knapsack(batch['status'], total_budget, **kwargs)
    indices = []
    for task_id, task_budget in enumerate(budget):
        if task_budget > 0:
            indices.extend([task_id] * task_budget)
    return batch.select_idxs(indices)

(*@\lstbg{blue!20}{gen\_batch = budget\_allocation(batch, total\_budget, **kwargs)}@*)
if rollout_balancing:
    indicies = np.random.shuffle(np.arange(len(batch['prompt'])))
    batch = batch.select_idxs(indicies)
batch = actor.generate_sequences(gen_batch)
batch = compute_rewards_and_advantages(batch)
(*@\lstbg{blue!20}{train\_dataset.update\_status(batch)}@*)
actor.update(batch)
\end{lstlisting}

%% file: appendix/additional_results.tex
\section{Additional Results}
\label{sec:additional_results}

\subsection{Visualization of Exploration Process}

\paragraph{Evolution of Prompts.} To illustrate the impact of exploration budgets on individual prompt learning dynamics, we track and visualize the learning trajectories of several randomly selected prompts from the training data in \cref{fig:prompt_learning_dynamics}. Each subplot corresponds to a unique prompt, identified by its index in the title. We observe that for several examples, our framework effectively allocates more exploration budget, leading to complete learning of the prompt (e.g., prompts in the first row, first column, and second row, first column). Conversely, some tasks remain highly challenging, where neither Knapsack-GRPO nor GRPO achieves satisfactory performance (e.g., the prompt in the third row, second column).

\begin{figure}[t]
    \centering
    \includegraphics[width=\linewidth]{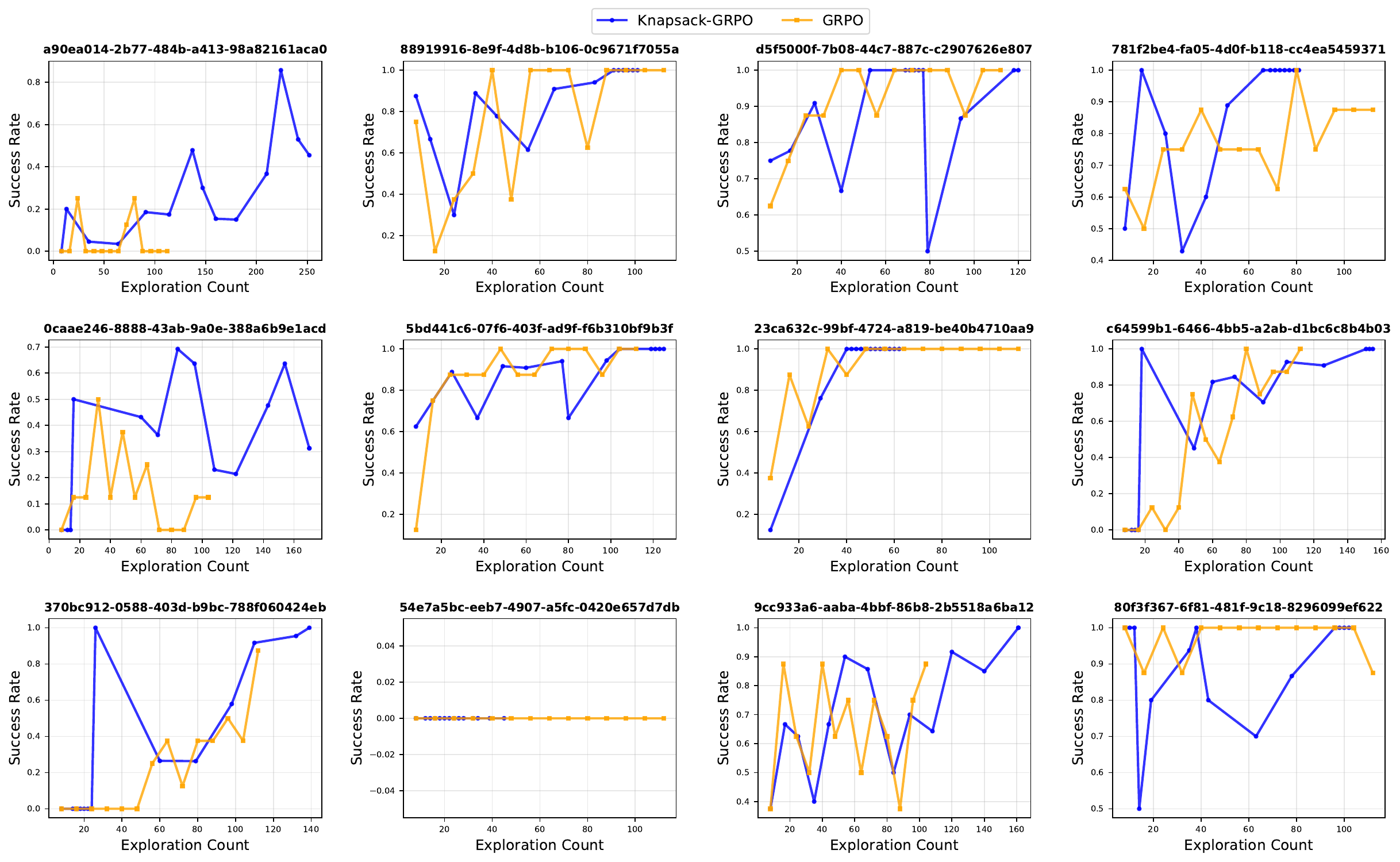}
    \caption{Learning dynamics of randomly selected prompts throughout training, comparing GRPO and Knapsack-GRPO. Each subplot shows the success rate evolution for a specific prompt.}
    \label{fig:prompt_learning_dynamics}
\end{figure}

\subsection{Training curves}
As references, the training curves for all models are displayed in Figures \ref{fig:dpsk_1.5b_training_curve}, \ref{fig:qwen3_4b_base_training_curve}, \ref{fig:qwen3_4b_instruct_training_curve}, and \ref{fig:qwen2.5_math_7b_training_curve}. Compared with the final results in Table \ref{tab:main_results}, these plots further show that Knapsack-GRPO delivers a rapid performance improvement  early in the training process. We also observe a few cases of performance degeneration, which points to the need for exploring more stable policy optimization techniques in future research.

\begin{figure}[htbp]
    \centering
    \includegraphics[width=\linewidth]{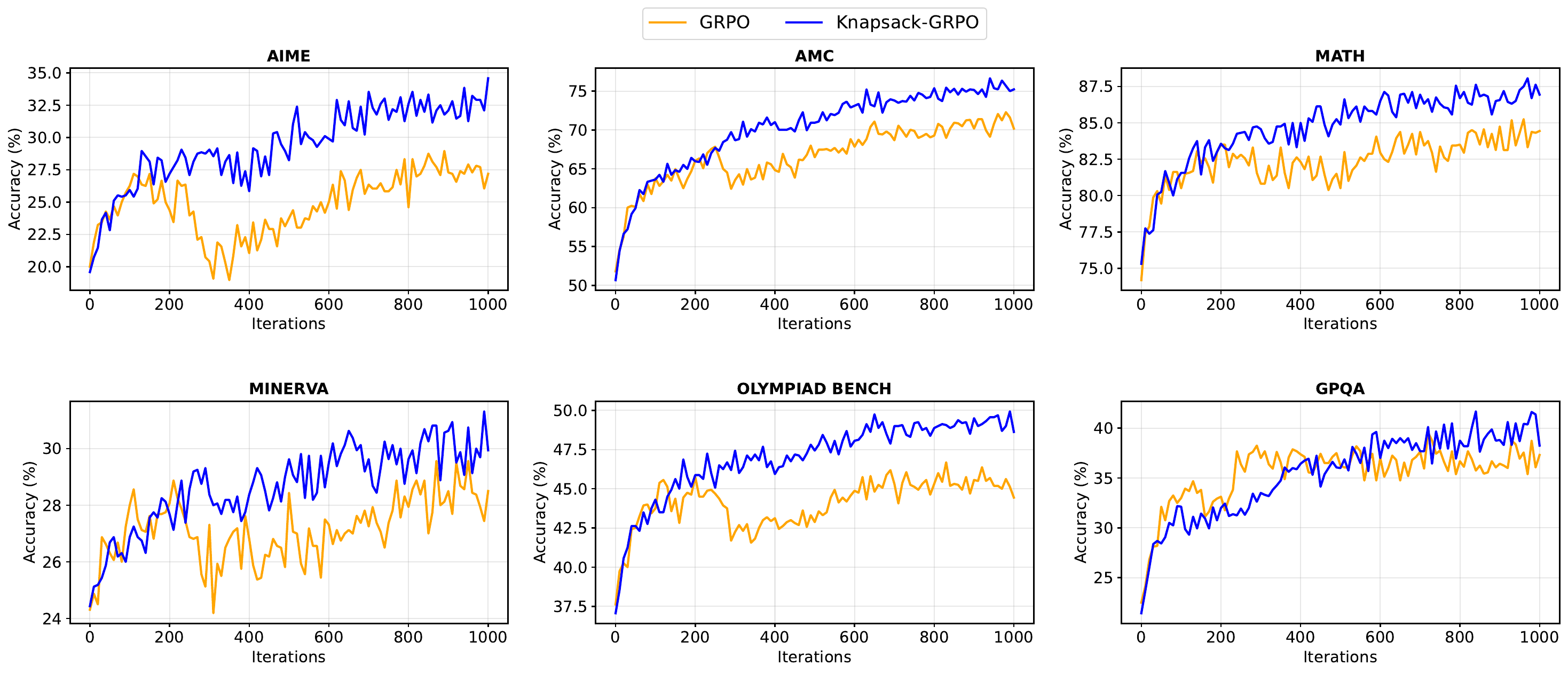}
    \caption{Evaluation performance of DPSK-R1-Distill-1.5B across training iterations.}
    \label{fig:dpsk_1.5b_training_curve}
\end{figure}

\begin{figure}[htbp]
    \centering
    \includegraphics[width=\linewidth]{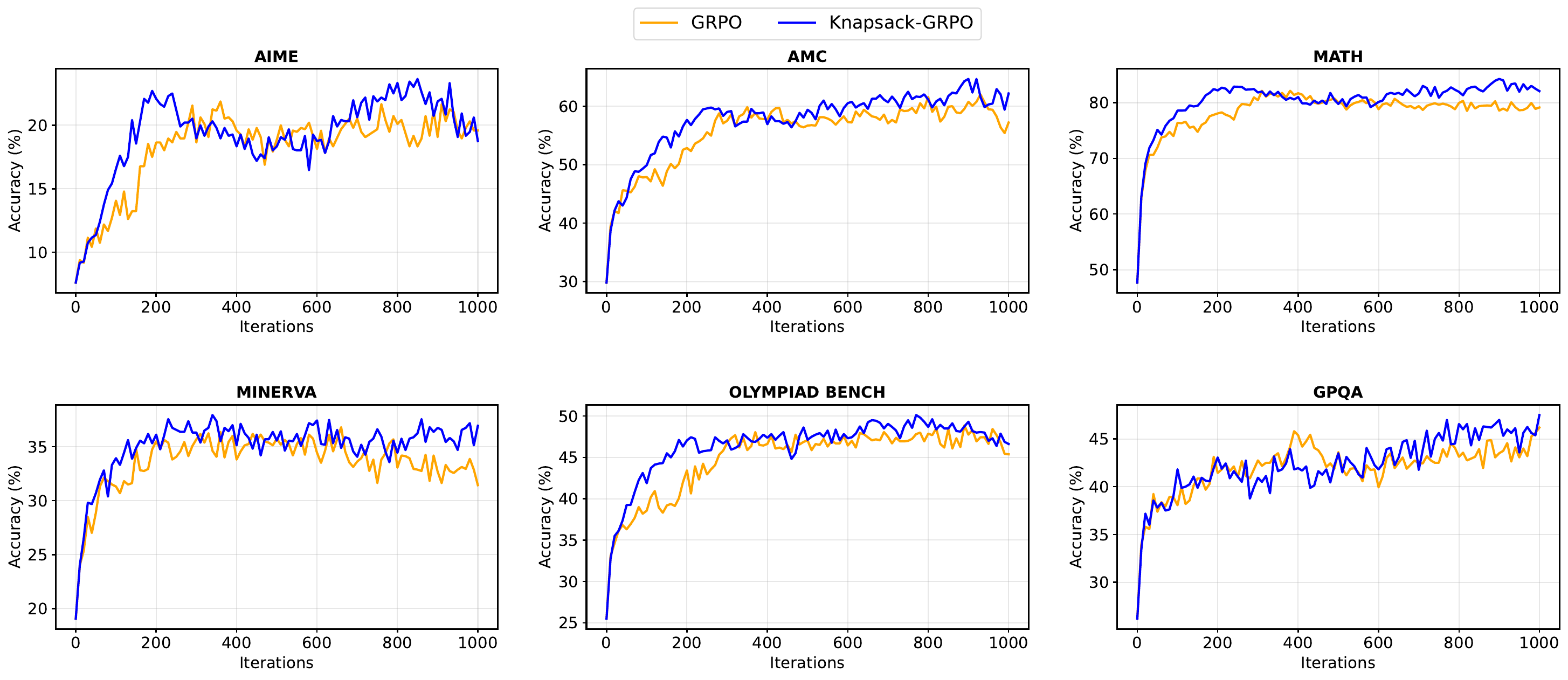}
    \caption{Evaluation performance of Qwen3-4B-Base across training iterations.}
    \label{fig:qwen3_4b_base_training_curve}
\end{figure}

\begin{figure}[htbp]
    \centering
    \includegraphics[width=\linewidth]{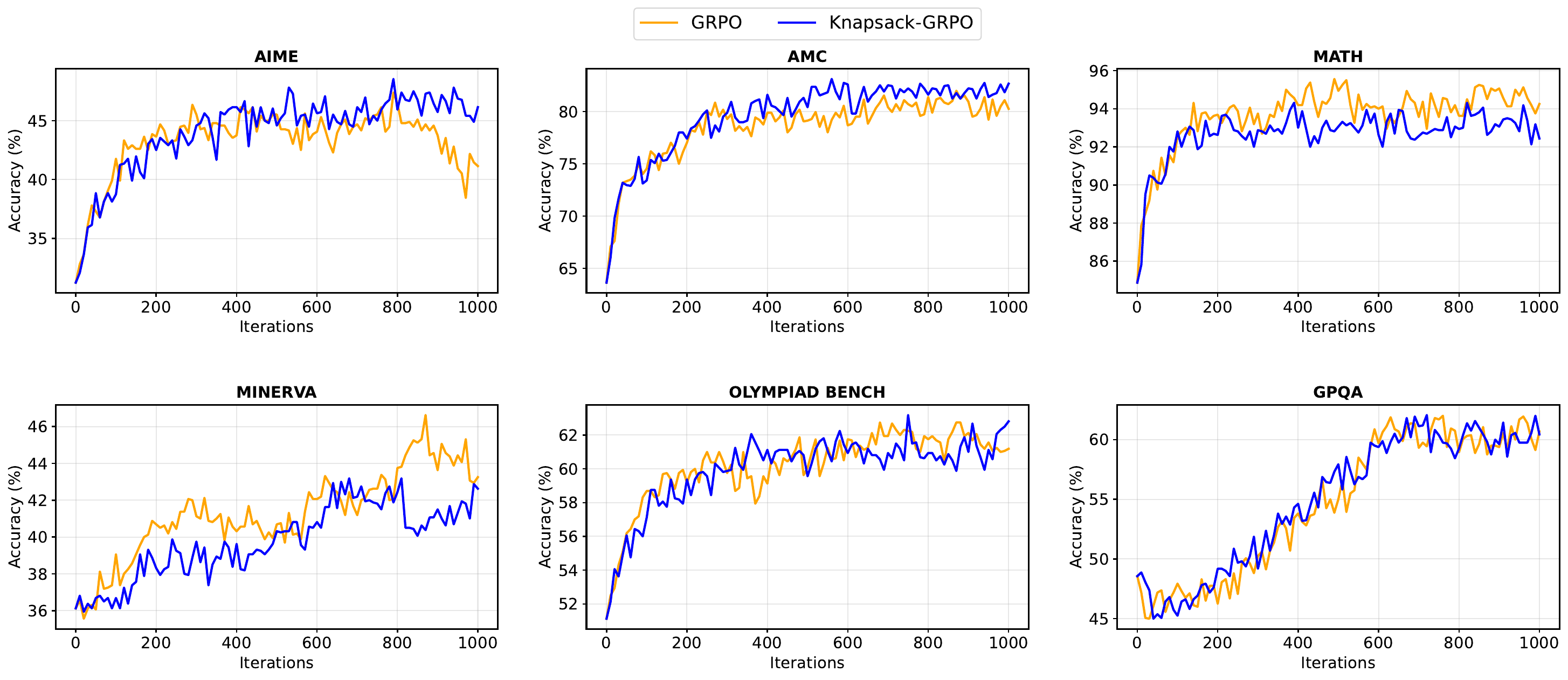}
    \caption{Evaluation performance of Qwen3-4B-Instruct across training iterations.}
    \label{fig:qwen3_4b_instruct_training_curve}
\end{figure}

\begin{figure}[htbp]
    \centering
    \includegraphics[width=\linewidth]{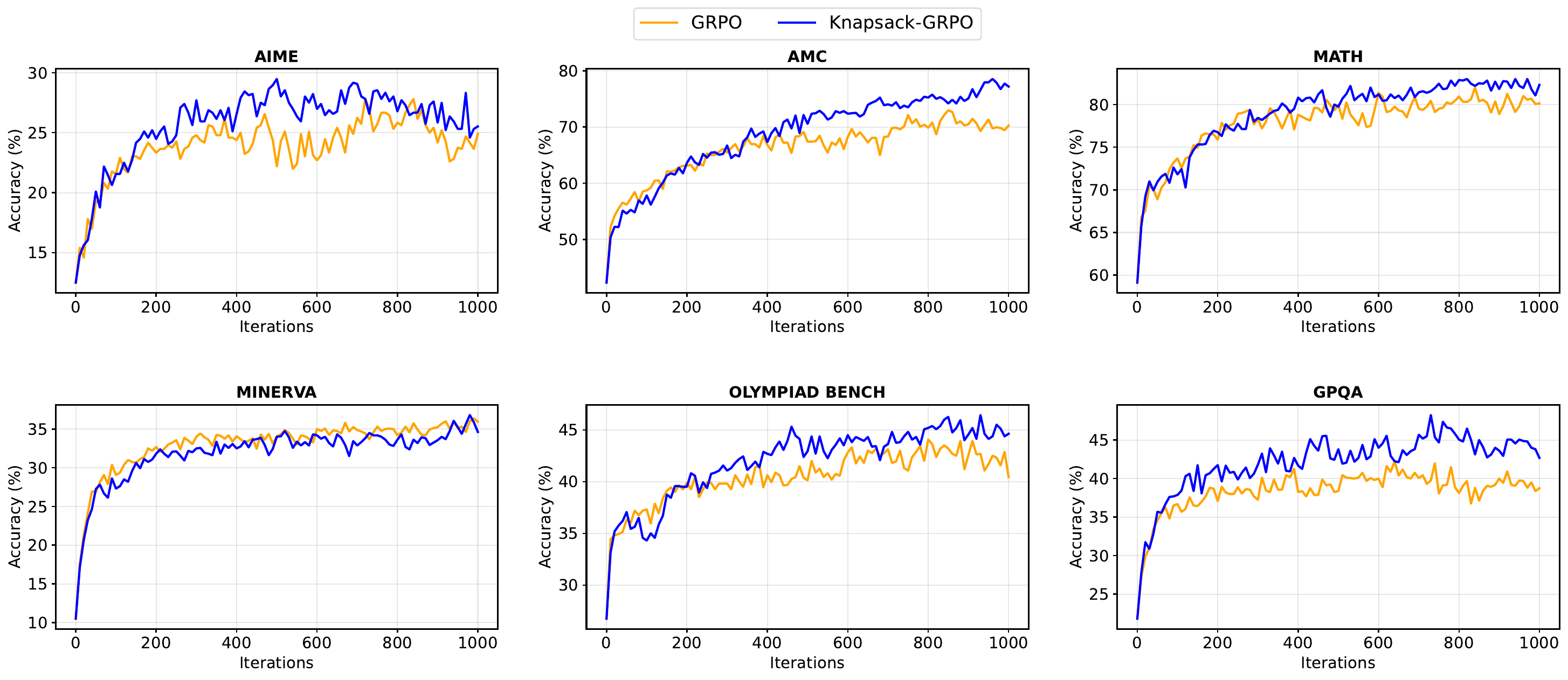}
    \caption{Evaluation performance of Qwen2.5-Math-7B across training iterations.}
    \label{fig:qwen2.5_math_7b_training_curve}
\end{figure}

\subsection{Ablation Studies}

\paragraph{Without Fallback Strategy.} In \ref{subsec:practical_considerations}, we introduced the \emph{fallback strategy}, which reallocates excess exploration budgets from already-solved prompts to those that remain unsolved. This prevents a common failure mode: difficult prompts may otherwise receive too few resources, while easy prompts are oversampled.

\begin{table}[htbp]
\centering
\caption{Comparison of budget allocation with and without fallback strategy.}
\label{tab:compare_fallback}
\begin{tabular}{c|cccc||cccc}
\toprule
& \multicolumn{4}{c||}{\textbf{With Fallback Strategy}} & \multicolumn{4}{c}{\textbf{Without Fallback Strategy}} \\
\cmidrule{2-9}
\textbf{Index} & \textbf{Success Rate} & \textbf{Cost} & \textbf{Assignment} & & \textbf{Success Rate} & \textbf{Cost} & \textbf{Assignment} & \\
\midrule
1 & 0.0 & $\infty$ & 29 & & 0.0 & $\infty$ & 2 & \\
2 & 0.9 & 22 & 23 & & 0.9 & 22 & 50 & \\
3 & 1.0 & 0 & 2 & & 1.0 & 0.0 & 2 & \\
4 & 1.0 & 0 & 2 & & 1.0 & 0.0 & 2 & \\
5 & 1.0 & 0 & 2 & & 1.0 & 0.0 & 2 & \\
6 & 1.0 & 0 & 2 & & 1.0 & 0.0 & 2 & \\
7 & 1.0 & 0 & 2 & & 1.0 & 0.0 & 2 & \\
8 & 1.0 & 0 & 2 & & 1.0 & 0.0 & 2 & \\
\bottomrule
\end{tabular}
\end{table}

A concrete example is shown in \cref{tab:compare_fallback} with $8$ prompts. Without the fallback strategy, the allocation assigns over 50 exploration units to a task with a success rate of $0.9$, while the unsolved task (success rate $0.0$) receives only 2 units. In contrast, with the fallback strategy, the unsolved task is assigned 29 units—substantially increasing its chance of making progress.

Empirically, this design proves crucial (\cref{fig:wo_fallback}). In our experiments with the Qwen2.5-Math-7B model, removing the fallback strategy led to unstable training, large performance fluctuations on benchmarks such as AMC and OlympiadBench, and overall degraded results. This result suggests that neglecting challenging examples during training weakens the reinforcement signal, ultimately harming the model’s ability to generalize.

\begin{figure}[t]
    \centering
    \includegraphics[width=\linewidth]{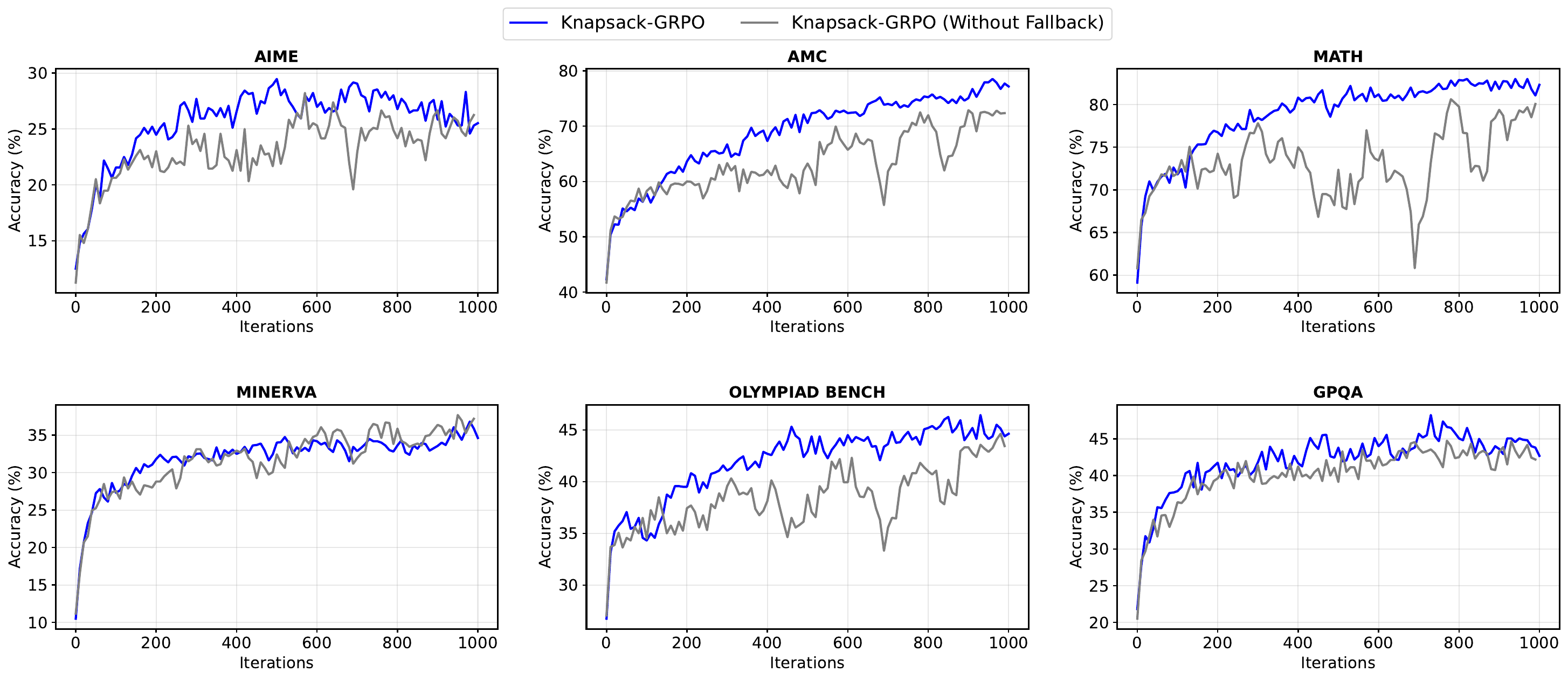}
    \caption{Effect of the fallback strategy. Without it, exploration budgets are disproportionately allocated to prompts with at least one successful trial, while unsolved tasks are largely ignored.}
    \label{fig:wo_fallback}
\end{figure}

\paragraph{Low and Up Bounds.} Our framework incorporates safeguards in the form of hyper-parameters $N_{\operatorname{low}}$ and $N_{\operatorname{up}}$, as defined in \cref{eq:knapsack_rl}. $N_{\operatorname{up}}$ is set to 128 primarily to facilitate faster computation of the knapsack optimization using dynamic programming; its specific value does not critically impact performance. Conversely, $N_{\operatorname{low}}$ is set to 2 to prevent degenerate allocation scenarios, particularly when success rates might be inaccurate, as elaborated in \cref{subsec:practical_considerations}. We present ablation results for these bounds in \cref{fig:low_upper_bounds}, which empirically support these design choices.

\begin{figure}[htbp]
    \centering
    \includegraphics[width=\linewidth]{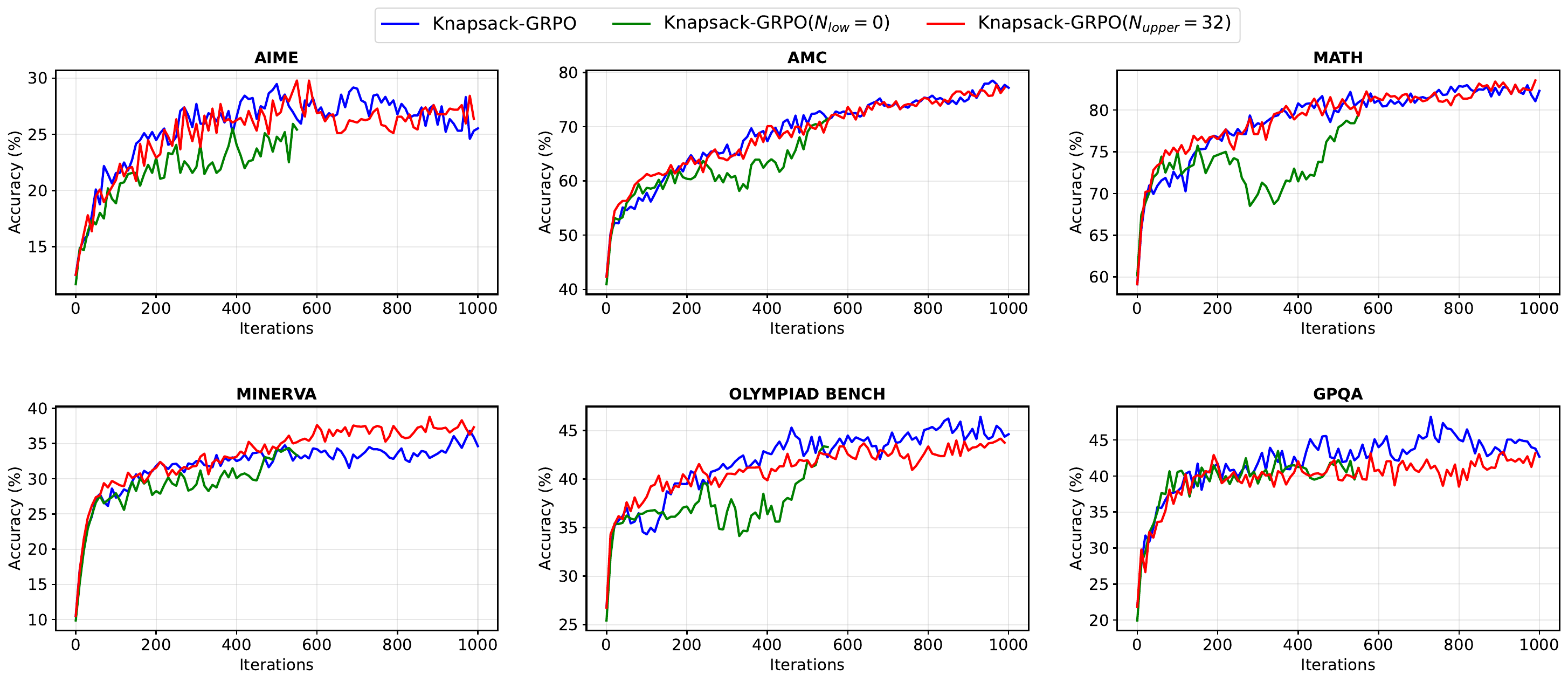}
    \caption{Ablation study on the impact of $N_{\operatorname{low}}$ and $N_{\operatorname{up}}$ constraints within the knapsack optimization framework.}
    \label{fig:low_upper_bounds}
\end{figure}

\subsection{Comparing with Dynamic Sampling in DAPO}

Dynamic sampling, a technique introduced in the DAPO paper \citep{yu2025dapo}, selects prompts with a mix of positive and negative rewards, filtering out those with exclusively positive or negative outcomes. This process is repeated until a target number of prompts is accumulated, a strategy that has been shown to be effective.

While effective, dynamic sampling operates on a different principle than our knapsack-based approach. Dynamic sampling aims to scale up effective \textbf{prompts}, while our method focuses on scaling up effective \textbf{responses}. Since these two approaches are parallel and can be combined, we conducted empirical studies to explore their synergy.

We evaluated the performance of these methods using two different metrics, as shown in the training curves in \cref{fig:dapo_exploration_iterations} and \cref{fig:dapo_gradient_iterations}. Because dynamic sampling requires multiple exploration steps to accumulate enough effective prompts for a single gradient update, we can analyze performance in two ways:
\begin{itemize}
    \item By exploration budget: \cref{fig:dapo_exploration_iterations} shows performance relative to the total number of exploration iterations. This measures how effectively total computation budget is converted into performance gains. We found that dynamic sampling boosts GRPO's performance on benchmarks like AIME and OLYMPIAD, improving the score from 45.2 to 46.2. When we combined dynamic sampling with our knapsack-based exploration, performance on the AMC benchmark improved significantly (from 69.8 to 73.0), resulting in a total performance of 46.5. This is slightly better than dynamic sampling alone but worse than our pure knapsack approach. We attribute this partially to the fact that knapsack-GRPO utilizes more gradient iterations, and therefore do not consider this a negative result.
    \item By gradient update iterations: \cref{fig:dapo_gradient_iterations} displays performance against the number of gradient updates. This metric assesses the value of each gradient update. The results clearly show that effective gradients, whether from dynamic sampling or our knapsack-based exploration, lead to greater performance gains for the same number of update iterations, which validates the core motivation behind both techniques.
\end{itemize}

\begin{figure}[htbp]
    \centering
    \includegraphics[width=0.9\linewidth]{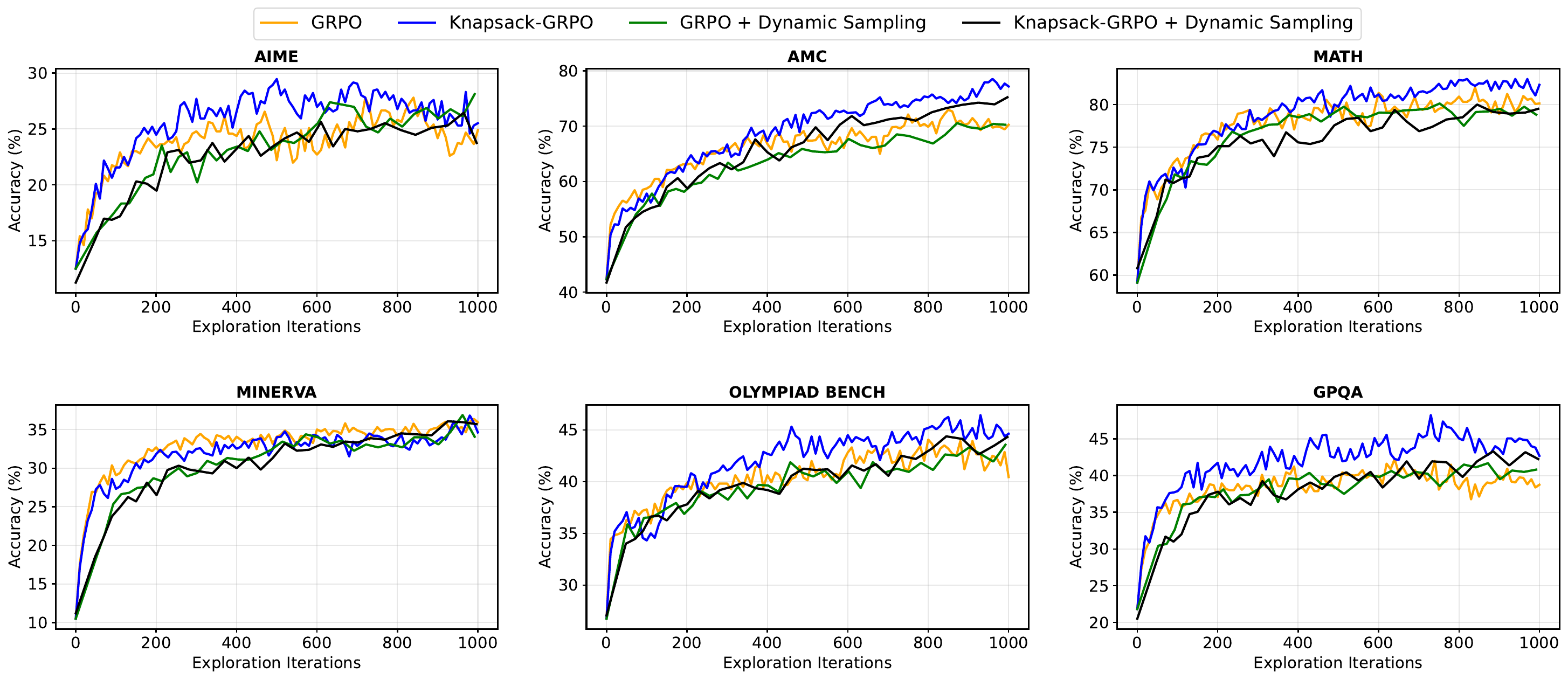}
    \caption{Performance of Qwen2.5-Math-7B relative to the number of exploration iterations, demonstrating how effectively the total computation budget is converted into performance gains.}
    \label{fig:dapo_exploration_iterations}
\end{figure}

\begin{figure}[htbp]
    \centering
    \includegraphics[width=0.9\linewidth]{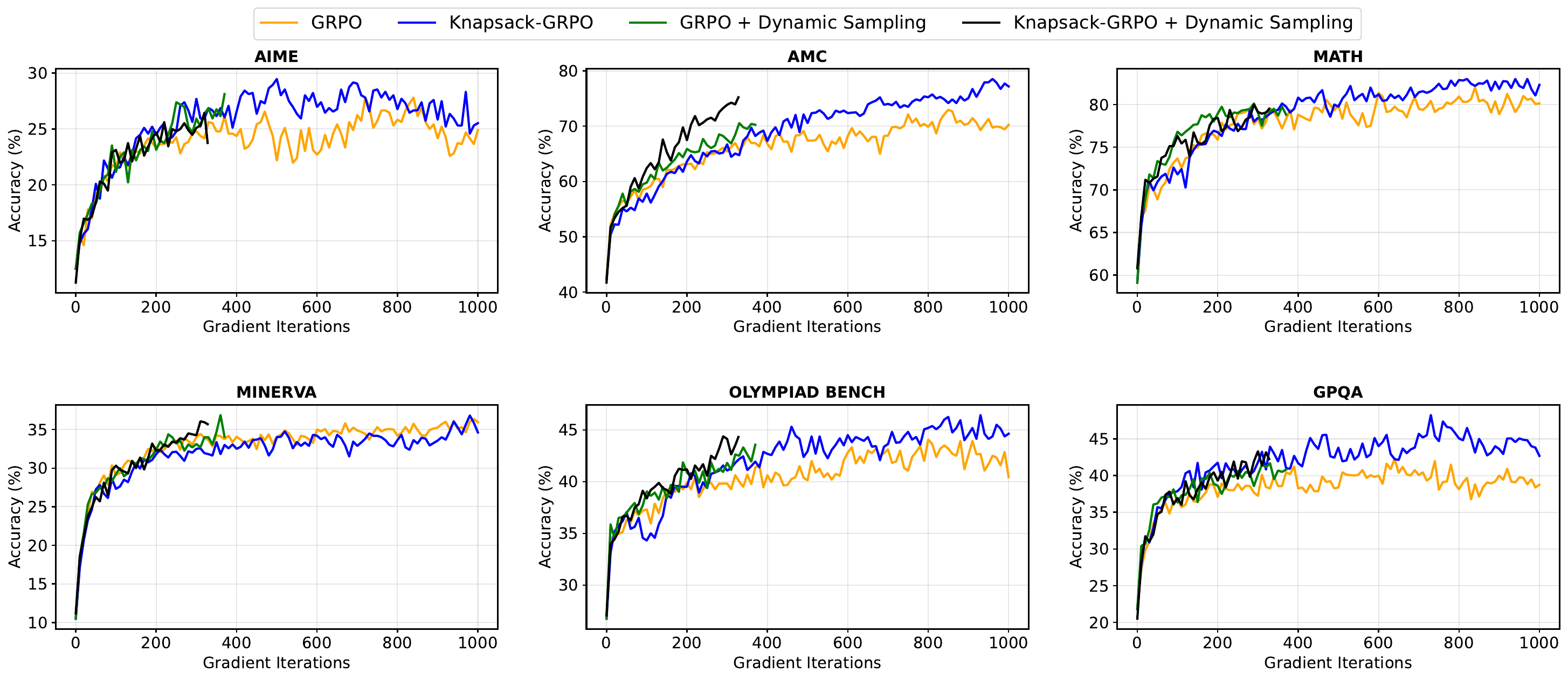}
    \caption{Performance of Qwen2.5-Math-7B as a function of the number of LLM gradient updates. This figure validates that effective gradients, derived from either dynamic sampling or the knapsack-based approach, lead to greater performance gains for the same number of updates.}
    \label{fig:dapo_gradient_iterations}
\end{figure}